\documentclass[sn-mathphys-num]{sn-jnl}%
\usepackage{lmodern}

\usepackage{graphicx}%
\usepackage{multirow}%
\usepackage{amsmath,amssymb,amsfonts}%
\usepackage{amsthm}%
\usepackage{mathrsfs}%
\usepackage[title]{appendix}%
\usepackage{xcolor}%
\usepackage{textcomp}%
\usepackage{manyfoot}%
\usepackage{booktabs}%
\usepackage{algorithm}%
\usepackage{algorithmicx}%
\usepackage{algpseudocode}%
\usepackage{listings}%

\usepackage{cleveref} %
\usepackage{subcaption} %
\usepackage{mathtools}
\usepackage{stmaryrd}
\usepackage{etoolbox} %

\RequirePackage{paralist}
\usepackage{enumitem}
\usepackage{multirow} %
\usepackage{array}
\usepackage{graphicx}
\usepackage{etoc}
\usepackage{listings}
\usepackage{eqparbox}
\usepackage{bbm}
\usepackage{soul}
\usepackage{lmodern}

\AtEndPreamble{%
\newtheorem{example}{Example} 
\newtheorem{theorem}{Theorem}
\newtheorem{lemma}[theorem]{Lemma} 
 
\newtheorem{corollary}[theorem]{Corollary}
\newtheorem{definition}[theorem]{Definition}

  \newtheorem{remark}[theorem]{Remark}
  \newtheorem{assumption}{Assumption}
  \newtheorem{condition}{Condition}
  \newtheorem{property}{Property}
  \theoremstyle{acmplain}
  \newenvironment{definition*}
                 {\pushQED{\qed}\definition}
                 {\popQED\enddefinition}  
  \newenvironment{example*}
                 {\pushQED{\qed}\example}
                 {\popQED\endremark}  
  \newenvironment{remark*}
                 {\pushQED{\qed}\remark}
                 {\popQED\endremark}  
  \newenvironment{assumption*}
                 {\pushQED{\qed}\assumption}
                 {\popQED\endremark}  
  \newenvironment{property*}
                 {\pushQED{\qed}\property}
                 {\popQED\endremark}  
}

\crefname{assumption}{Assumption}{Assumptions}
\crefname{figure}{Fig{.}}{Figs{.}}%
\crefname{table}{Table}{Tables}
\crefname{definition}{Definition}{Definitions}
\crefname{theorem}{Theorem}{Theorems}
\crefname{lemma}{Lemma}{Lemmas}
\crefname{proposition}{Proposition}{Propositions}
\crefname{corollary}{Corollary}{Corollaries}
\crefname{problem}{Problem}{Problems}
\crefname{example}{Example}{Examples}
\crefname{fact}{Fact}{Facts}
\crefname{conjecture}{Conjecture}{Conjectures}
\crefname{remark}{Remark}{Remarks}
\crefname{condition}{Condition}{Conditions}
\crefname{requirement}{Requirement}{Requirements}
\crefname{enumi}{}{}
\crefname{equation}{Eq{.}}{Eqs{.}}
\crefname{section}{Section}{Sections}

\newcommand{\mathboldcommand}[1]{\mathbb{#1}}

\newcommand{\bbN}{\mathboldcommand{N}}

\newcommand{\bbQ}{\mathboldcommand{Q}}
\newcommand{\bbR}{\mathboldcommand{R}}

\newcommand{\bbZ}{\mathboldcommand{Z}}

\newcommand{\bfa}{\mathbf{a}}
\newcommand{\bfb}{\mathbf{b}}
\newcommand{\bfc}{\mathbf{c}}

\newcommand{\bfp}{\mathbf{p}}

\newcommand{\bfv}{\mathbf{v}}
\newcommand{\bfw}{\mathbf{w}}
\newcommand{\bfx}{\mathbf{x}}

\newcommand{\bfz}{\mathbf{z}}

\usepackage{euscript}
\newcommand{\mathcalcommand}[1]{\mathcal{#1}}

\newcommand{\mcB}{\mathcalcommand{B}}
\newcommand{\mcC}{\mathcalcommand{C}}

\newcommand{\mcG}{\mathcalcommand{G}}

\newcommand{\mcI}{\mathcalcommand{I}}

\newcommand{\mcN}{\mathcalcommand{N}}

\newcommand{\mcS}{\mathcalcommand{S}}

\newcommand{\mcV}{\mathcalcommand{V}}

\newcommand{\mcX}{\mathcalcommand{X}}

\DeclareMathAlphabet{\mathpzc}{T1}{pzc}{m}{it}

\newcommand*{\aff}{{\mathrm{aff}}}
\newcommand*{\defeq}{\triangleq}
\newcommand*{\round}[1]{ \left\lceil{#1}\right\rfloor}

\DeclareMathOperator*{\argmin}{arg\,min}

\newcommand{\Qps}{\bbQ_{p,s}}
\newcommand{\Qis}{\bbQ_{\infty,s}}

\newcommand{\tcv}[1]{{{\color{black}{#1}}}}

\newcommand*{\relu}{\mathrm{ReLU}}
\newcommand*{\SiLU}{\mathrm{SiLU}}
\newcommand*{\SoftPlus}{\mathrm{SoftPlus}}
\newcommand*{\Mish}{\mathrm{Mish}}
\newcommand*{\GeLU}{\mathrm{GELU}}
\newcommand*{\erf}{\mathrm{erf}}
\newcommand*{\Sigmoid}{\mathrm{Sigmoid}}
\newcommand*{\ELU}{\mathrm{ELU}}

\newcommand*{\elu}{\mathrm{ELU}}

\newcommand*{\hardtanh}{\mathrm{Hardtanh}}

\newcommand*{\indc}[2]{\mathbbm{1}_{#1}\left({#2}\right)}
\newcommand{\Mod}[1]{\ (\mathrm{mod}\ #1)}
\newcommand{\modd}[1]{~\mathrm{mod}~#1}

\DeclareMathOperator{\sgn}{sgn}

\begin{document}

\title{On Expressive Power of Quantized Neural Networks under Fixed-Point Arithmetic}

\author[1]{\fnm{Yeachan} \sur{Park}}\email{ychpark@sejong.ac.kr}
\equalcont{These authors contributed equally to this work.}

\author[2]{\fnm{Sejun} \sur{Park}}\email{sejun.park000@gmail.com}\equalcont{These authors contributed equally to this work.}

\author*[3]{\fnm{Geonho} \sur{Hwang}}\email{hgh2134@gist.ac.kr}

\affil[1]{\orgdiv{Department of Mathematics and Statistics}, \orgname{Sejong University}, \orgaddress{\city{Seoul}, \postcode{05006}, \country{Republic of Korea}}}

\affil[2]{\orgdiv{Department of Artificial Intelligence}, \orgname{Korea University}, \orgaddress{ \city{Seoul}, \postcode{02841}, \country{Republic of Korea}}}

\affil[3]{\orgdiv{Department of Mathematical Sciences}, \orgname{Gwangju Institute of Science and Technology}, \orgaddress{\city{Gwangju}, \postcode{61005}, \country{Republic of Korea}}}

\abstract{Existing works on the expressive power of neural networks typically assume real parameters and exact operations.
In this work, we study the expressive power of quantized networks under discrete fixed-point parameters and inexact fixed-point operations with round-off errors. 
We first provide a necessary condition and a sufficient condition on fixed-point arithmetic and activation functions for quantized networks to represent all fixed-point functions from fixed-point vectors to fixed-point numbers. 
Then, we show that various popular activation functions satisfy our sufficient condition, e.g., Sigmoid, ReLU, ELU, SoftPlus, SiLU, Mish, and GELU.
In other words, networks using those activation functions are capable of representing all fixed-point functions.
We further show that our necessary condition and sufficient condition coincide under a mild condition on activation functions: e.g., for an activation function $\sigma$, there exists a fixed-point number $x$ such that $\sigma(x)=0$. Namely, we find a necessary and sufficient condition for a large class of activation functions.
We lastly show that even quantized networks using binary weights in $\{-1,1\}$ can also represent all fixed-point functions for practical activation functions.}



\maketitle

\section{Introduction}\label{sec:intro}
In the theory of neural networks, a class of results called the \emph{universal approximation theorem} state that for any continuous function $f^*$ on a compact domain $\mcC\subset\bbR^d$ and for any error bound $\varepsilon>0$, there exists a neural network $f$ such that $f^*$ $\sup_{x\in\mcC}\|f^*(x)-f(x)\|_\infty\le\varepsilon$. 
Such results have been established for various network architectures (e.g., multi-layer perceptron \cite{cybenko89,Hornik89,leshno1993multilayer,pinkus99,yarotsky2017error}, convolutional neural networks \cite{zhou2020universality,shen2022approximation}, transformer \cite{Yun2020Are,yun2020n,alberti2023sumformer,kajitsuka2024are}) and various error measures (e.g., $L^p$ norm \cite{Hornik89,park21,kim2024minimum,shin2025minimum} and uniform norm \cite{cybenko89,leshno1993multilayer,kidger2020universal}). 
Furthermore, they are often considered proof of the following common belief:
$$\textit{Neural networks in computers can approximate any continuous function.}$$
However, such a common belief does not naturally follow from the existing results. This is because existing theory uses real numbers and exact mathematical operations in its analyses, but computers can only represent a finite subset of reals and perform inexact operations due to the bounded precision and running time. 

Such a discrepancy between theory and practice becomes more serious
when networks are implemented with low-precision operations, where the round-off error is large. 
In particular, with the recent exponential growth of the number of parameters in state-of-the-art networks, low-precision parameters and operations are widely adopted \cite{hubara2018quantized,ligptaq,spinquant}.
Network quantization is a popular method that can reduce the memory and computation costs of networks by using low-precision fixed-point parameters and low-cost integer operations \cite{huai2023crossbar,jacob2018quantization,jin2021f8net,li2023vit,sari2021irnn,wang2022niti,yao2021hawq,zhao2021efficient}.  %
Surprisingly, although quantized networks using fixed-point arithmetic have discrete parameters and non-negligible round-off errors in their evaluation, they have successfully reduced memory and computation costs while preserving the performance of their unquantized counterparts.

Only a few works have investigated the expressive power of networks using discrete parameters and/or inexact machine operations. %
For example, Ding et al.\ \cite{ding2018universal} show networks using quantized (i.e., discrete) weights and exact mathematical operations can universally approximate.
In addition, Gonon et al.\ \cite{gonon2023approximation} analyze the approximation error incurred by quantizing real network parameters through nearest rounding. 
However, almost all existing works consider operations without error (i.e., exact), and thus, are not applicable to quantized networks using fixed-point operations.
The only exceptions are recent works that show networks under floating-point arithmetic (i.e., floating-point parameters and floating-point operations) can represent almost all floating-point functions \cite{park2024expressive,hwang2025floating,hwang2025floatingpoint}.
Nevertheless, these results assume floating-point arithmetic only, and hence, it does not apply to quantized networks using fixed-point arithmetic. Namely, what modern quantized networks using general activation functions can or cannot express is still unknown.

In this work, we study the expressive power of quantized networks using fixed-point arithmetic only.
Specifically, we consider networks using $(p+1)$-bit fixed-point arithmetic that consists of the set of fixed-point numbers
\begin{align}
\Qps\defeq\left\{k/s:k\in\bbZ, -2^p+1\le k\le2^p-1\right\} ,
\end{align}
for some scaling factor $s\in\bbN$
and the fixed-point operations derived from rounding $\round{x}$ of $x\in\bbR$, which denotes an element in $\Qps$ closest to $x$ (see \cref{sec:fixed-point} for the precise definition and the tie-breaking rule).
Given such fixed-point arithmetic, a pointwise activation function $\sigma$, and affine transformations $\rho_1,\dots,\rho_L$, we consider a ``$\sigma$ quantized network'' $f:\Qps^d\to\Qps$ defined as
\begin{align}
f(\bfx)\defeq\round{\rho_L}\circ\round{\sigma}\circ\round{\rho_{L-1}}\circ\round{\sigma}\circ\cdots\circ\round{\sigma}\circ\round{\rho_{1}}(\bfx),\label{eq:nn-intro}
\end{align}
where $\round{\rho_l}$ and $\round{\sigma}$ denote the functions that round the output elements of $\rho_l$ and $\sigma$ to $\Qps$, respectively (refer to \cref{sec:nn} for precise definition).
\tcv{Namely, our computational model quantizes the output of each activation function and affine transformation, but does not quantize operations inside them (e.g., additions and multiplications in affine transformations).} 
We note that such quantized networks have been used in the network quantization literature \cite{jacob2018quantization,yao2021hawq}. Under this setup, we investigate the following \emph{universal representability}:
\begin{gather*}
\textit{Can quantized networks represent all functions from $\Qps^d$ to $\Qps$?}
\end{gather*}
Our contributions can be summarized as follows.
\begin{itemize}
\item We first provide a necessary condition on activation functions and fixed-point arithmetic (i.e., $\Qps$) for universal representation of quantized networks in \cref{thm:necessity}. Unlike classical results that show networks using any non-affine continuous activation function can universally approximate under real parameters and exact mathematical operations \cite{leshno1993multilayer,kidger2020universal}, our necessary condition shows that quantized networks using rounded versions of some non-affine continuous functions cannot universally represent.
\item We then provide a sufficient condition on activation functions and $\Qps$ for universal representation in \cref{thm:sufficiency}. 
We show that rounded versions of various practical activation functions such as $\Sigmoid$, $\relu$, $\elu$, $\SoftPlus$, $\SiLU$, $\Mish$, and $\GeLU$\footnote{See \cref{sec:setup} for the definitions of activation functions.} satisfy our sufficient condition; that is, practical quantized networks are capable of performing a given target task. 
Interestingly, the identity activation function (i.e., $\sigma(x)=x$) also satisfies our sufficient condition, i.e., it is capable of universal representation unlike networks using real parameters and exact mathematical operations. 
\tcv{This is because rounded affine transformations ($\round{\rho_i}$ in \cref{eq:nn-intro}) are non-affine under exact operations in general due to the round-off error.}
We note that a similar observation has been recently made in networks using floating-point arithmetic \cite{hwang2025floating,hwang2025floatingpoint}.
\item We show that under a mild condition on activation functions (e.g., there exists $x$ such that $\sigma(x)=0$), our necessary condition coincides with our sufficient condition (\cref{cor:nece_suff,lem:nece_suff}).
This implies that for a large class of activation functions, our results (\cref{thm:necessity,thm:sufficiency}) provide a necessary and sufficient condition for universal representation.
\item We further extend our results to quantized networks with binary weights, i.e., all weights in the networks are in $\{-1,1\}$, and show that quantized networks with binary weights can universally represent for rounded versions of $\Sigmoid$, $\relu$, $\elu$, $\SoftPlus$, $\SiLU$, $\Mish$, and $\GeLU$. This setup has been widely studied in network quantization literature due to the low multiplication cost with $1$ and $-1$ \cite{ma2024era, wang2023bitnet}.
\item %
We show that a na\"ive quantization of real parameters in a network may incur a large error; hence, existing universal approximation results do not directly extend to quantized networks. We also quantitatively analyze the size of networks for approximating a target function using our results.
\end{itemize}

\subsection{Organization}
The problem setup and relevant notations are presented in Section \ref{sec:setup}.
In \cref{sec:main_results}, we present our principal findings on universal representation property of quantized networks.
Specifically, we provide necessary and/or sufficient conditions on activation function and fixed-point arithmetic for universal representation; we then extend these results to networks with binary weights.
We next discuss a na\"ive quantization method,  quantitatively analyze our results, and compare our results with results under floating-point arithmetic in \cref{sec:discussions}.
We provide formal proofs of our findings in \cref{sec:proofs} and conclude our paper in \cref{sec:conclusion}.

\section{Problem setup and notations}\label{sec:setup}

\subsection{Notations}\label{sec:notation}
We introduce frequently used notations here.
We use $\mathbb{N}$, $\bbZ$, $\mathbb R$, and $\mathbb{R}_{\ge 0}$ to denote the set of natural numbers, the set of integers, the set of real numbers, and the set of positive real numbers, respectively. We also use $\mathbb{N}_0\defeq\mathbb{N}\cup \{0\}$. 
For $n,m\in\bbN_0$, we define $[n]\defeq \left\{ 1,2,\dots, n\right\}$, i.e., $[0]=\emptyset$.  %
For $a,b\in \bbR$, an interval $[a,b]$ is defined as $[a,b]\defeq \{x\in \bbR: a\le x\le b\}$. 
We generally use $a, b, c, \dots$ to represent scalar values and $\bfa,\bfb,\bfc,\dots$ to denote column vectors.
\tcv{For a natural number $r\in \bbN$ and integers $a,b$, $a\equiv b \Mod{r}$ if and only if $r|a-b$, which means that their difference is a multiple of $r$.
We also use$\modd{r}$ as the remainder operator from $\bbZ$ to $[r-1]\cup\{0\}$.
Thus, $c = a \modd{r}$ is the unique number satisfying $c\in [r-1]\cup\{0\}$ and $c\equiv a \Mod{r}$.  }

For a vector $\bfx\in\bbR^n$ and an index $i\in[n]$, we use $x_i$ to represent the $i$-th coordinate of $\bfx$.
Likewise, for a function $f:\bbR^n\to\bbR^m$, we use $f(\bfx)_i$ to denote the $i$-th coordinate of $f(\bfx)$.
For a function $\sigma:\bbR\to\bbR$, we often use $\sigma$ with a vector-valued input (e.g., $\sigma(\bfx)$ for some $\bfx \in\bbR^n$) to denote its coordinate-wise application (i.e., $\sigma(\bfx)=(\sigma(x_1),\dots,\sigma(x_n))$).
We use the big-O notation to hide absolute constants only.

For a vector $\bfx\in\bbR^n$, $\dim(\bfx)$ denotes the dimensionality of the vector $\bfx$, i.e., $\dim(\bfx)=n$.
For a set $A$, $|A|$ denotes the cardinality of $A$.
For a vector $\bfx\in\bbR^n$ and a set $\mcS\subset \bbR^n$, we define an \emph{indicator function} $\indc{\mcS}{\bfx}$ as
\begin{align*}
\indc{\mcS}{\bfx}\defeq \begin{cases}
1~&\text{if}~\bfx\in\mcS,\\
0~&\text{if}~\bfx\notin\mcS.
\end{cases}
\end{align*}
We define the \emph{affine transformation} as follows:
for $n\in\bbN$, $k\in[n]$, $\bfx = (x_1, \dots, x_n)\in \bbR^{n}$, $\bfw = (w_1,\dots, w_k)\in \bbR^k$, $b\in\bbR$, and $\mcI = \{i_1, \dots, i_k\}\subset [n]$ with $i_1<\cdots<i_k$, $\aff(\,\cdot\,;\bfw,b,\mcI): \bbR^{n}\rightarrow \bbR$ is defined as
\begin{equation}
    \aff(\bfx; \bfw,b, \mcI)\defeq b+\sum_{j=1}^kx_{i_j}w_j.\label{eq:affine-def}
\end{equation}
We lastly provide mathematical definitions of popular activation functions:
\begin{itemize}
    \item $\Sigmoid(x) \defeq \frac{1}{1+e^{-x}}$,
    \item $\relu(x) \defeq \max(0,x)$,
        \item $        \ELU(x) \defeq \begin{cases}
            x  &\text{ if } x\ge 0,
            \\ \exp(x)-1 &\text{ if } x< 0,
        \end{cases}$
        \item $\SiLU(x) \defeq \frac{x}{1 + e^{-x}}$, %
        \item $\SoftPlus(x) \defeq \log(1+ \exp(x))$,
        \item $\Mish(x) \defeq x\tanh(\SoftPlus(x))$, %
        \item $\GeLU(x) \defeq \frac{x}{2}\left(1+ \erf\left(\frac{x}{\sqrt{2}}\right) \right) = \frac{x}{2}\left(1+ \frac{2}{\sqrt{\pi}} \int_0^{x/ \sqrt{2} } e^{-t^2}dt  \right)$,
        \item $\hardtanh(x) \defeq \begin{cases}
            1  &\quad \text{ if } x\ge 1, \\
            x  &\quad \text{ if } 0< x <1,\\
            0  &\quad \text{ if } x\le 0.
        \end{cases}$
\end{itemize}
\tcv{The graphs of these activation functions are depicted in \cref{fig:activation_functions}.}
\begin{figure}
    \centering
    \includegraphics[width=\linewidth]{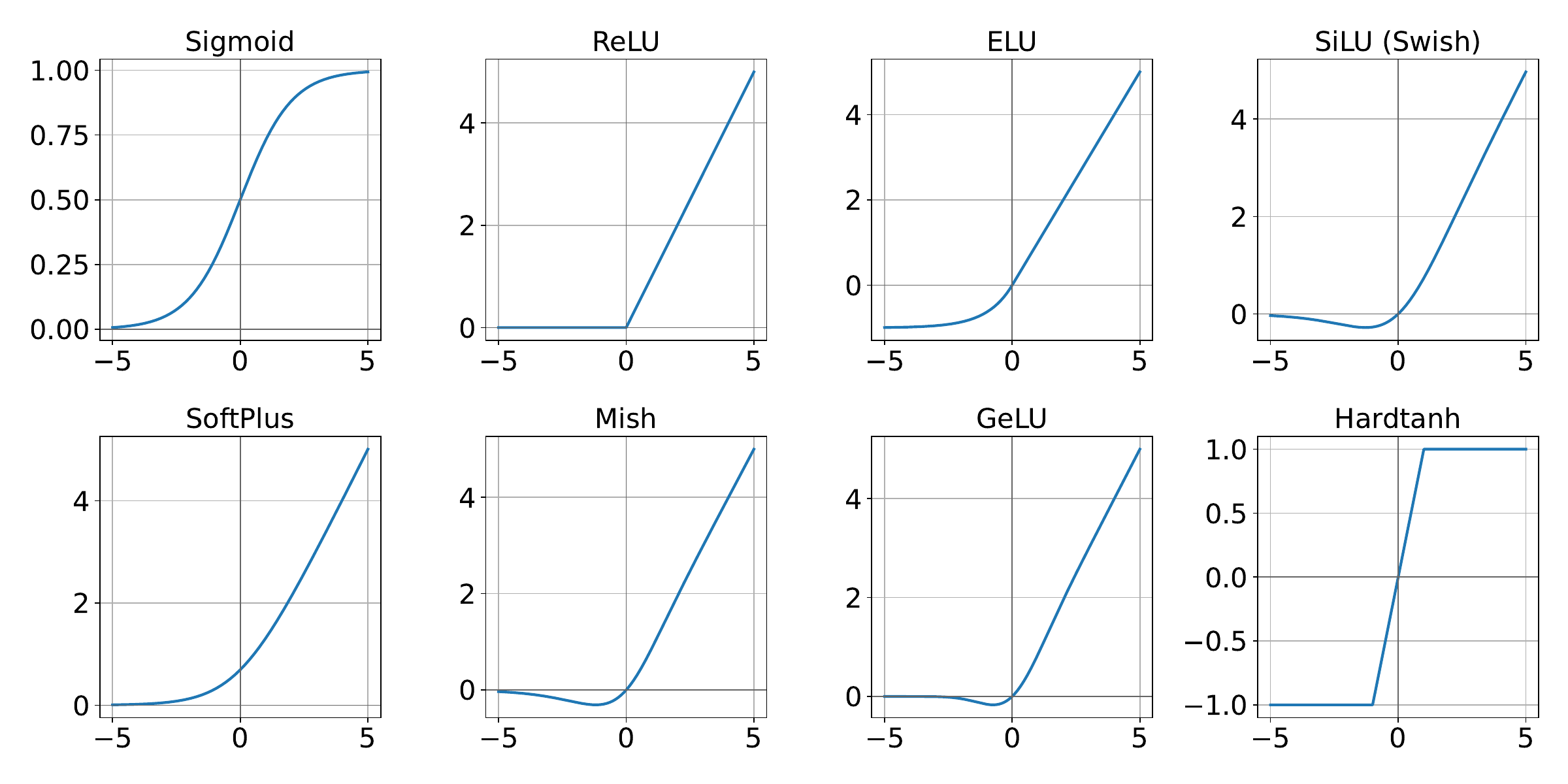}
    \caption{Plots of the activation functions}
    \label{fig:activation_functions}
\end{figure}
\subsection{Fixed-point arithmetic}\label{sec:fixed-point}
In this section, we introduce fixed-point arithmetic that we focus on; such arithmetic is used in \cite{jacob2018quantization,yao2021hawq}. %
In particular, we consider the following set of fixed-point numbers: for $p\in\bbN\cup\{\infty\}$ and $s\in \mathbb{N}$, 
    \begin{equation}
        \Qps\defeq
        \begin{cases}
            \left\{ q/s : q\in [-2^{p}+1, 2^{p}-1]\cap\bbZ\right\}~&\text{if}~p<\infty,\\
            \left\{ q/s : q\in \bbZ\right\}~&\text{if}~p=\infty.
        \end{cases}\label{eq:fixedpoint}
    \end{equation}
Throughout this paper, we define $q_{p,s,\max}$ $\defeq$ $\max \Qps$ = $\frac{2^p-1}s$, and assume $q_{p,s,\max}$  $\ge 1$ , i.e., $-1,1\in\Qps$ and $s\le 2^p-1$.
If $p,s$ are apparent from the context, we drop $p,s$ and use $q_{\max}$ to denote $q_{p,s,\max}$. %
    
For any real number $x\in \bbR$, we define the rounding operation $\round{\cdot}_{\Qps}$ as follows: %
\begin{equation*}
    \round{x}_{\Qps} \defeq \argmin_{y \in \Qps} |x-y|.
\end{equation*}
When ties occur (i.e., there are $u,v\in\Qps$ such that $u\ne v$ and $|x-u|=|x-v|=\min_{y \in \Qps} |x-y|$), we choose the number with the larger absolute value in the set $\argmin_{y \in \Qps} |x-y|$ (i.e., $\round{x}_{\Qps}=u$ if $|u|>|v|$) following  \cite{jacob2018quantization}.\footnote{Such a rounding scheme is often referred to as ``away from zero''.} %
To simplify notation, we frequently omit $\Qps$ and use $\round{x}$ to denote $\round{x}_{\Qps}$ if $\Qps$ is clear.

Given a function $f:\bbR^n\to\bbR^m$ and $\Qps$, we define its \emph{quantized version} $\round{f}_{\Qps}:\bbR^n\to\Qps^m$ with respect to $\Qps$ as follows: for $\bfx\in\bbR^n$
\begin{align}
\round{f}_{\Qps}(\bfx) \defeq \left(\round{f(\bfx)_1}_{\Qps},\dots,\round{f(\bfx)_m}_{\Qps}\right)\label{eq:round-fcn}.
\end{align}
Here, we frequently omit $\Qps$ and use $\round{f}$ to denote $\round{f}_{\Qps}$ if $\Qps$ is apparent from the context.

\subsection{Neural networks}\label{sec:nn}

Let $L\in \bbN$ be the number of layers, $N_0=d\in\bbN$ be the input dimension, $N_L\in\bbN$ be the output dimension, and $N_\ell\in\bbN$ be the number of hidden neurons (i.e., hidden dimension) at layer $\ell$ for all $\ell\in[L-1]$. For each $l\in[L]$ and $i\in[N_l]$, let %
$\mcI_{l,i}\subset[N_{l-1}]$ be the set of indices of hidden neurons in the layer $l-1$ that are used for computing the $i$-th neuron of the layer $l$ via some affine map characterized by parameters $\bfw_{l,i}\in\bbR^{|\mcI_{l,i}|}$, and $b_{l,i}\in\bbR$ (see \cref{eq:affine-def}).
Let $\mathcal{I}\defeq(\mathcal I_{1,1},\dots,\mathcal I_{1,N_1},\dots,\mathcal I_{L,1},\dots,\mathcal I_{L,N_L})$ and $\theta\in\bbR^I$ be the concatenation of all $\bfw_{l,i}$, and $b_{l,i}$ where the \emph{total number of parameters} $I$ is defined as
\begin{align*}
I&\defeq\sum_{l=1}^L\sum_{i=1}^{N_l}(\dim(b_{l,i})+\dim(\bfw_{l,i}))=\sum_{l=1}^LN_l+\sum_{l=1}^L\sum_{i=1}^{N_l}|\mcI_{l,i}|.
\end{align*}
We define a \emph{neural network} $g_{\theta,\mcI}:\bbR^{d}\to\bbR$ using $\sigma:\bbR\to\bbR$ as its activation function as follows: for $\bfx\in\bbR^{d}$,
\begin{align}
g_{\theta,\mcI}(\bfx)\defeq\rho_L\circ\sigma\circ\rho_{L-1}\circ\sigma\circ\cdots\circ\rho_2\circ\sigma\circ\rho_1(\bfx),\label{eq:nn-def}
\end{align}
where $\rho_l:\bbR^{N_{l-1}}\to\bbR^{N_l}$ is defined as
\begin{equation*}
\rho_l(\bfz)\defeq\big(\aff(\bfz;\bfw_{l,1},b_{l,1},\mcI_{l,1}),\dots,\aff(\bfz;\bfw_{l,N_l},b_{l,N_l},\mcI_{l,N_l})\big), %
\end{equation*}
for all $l\in[L]$ (see \cref{eq:affine-def} for the definition of $\aff$). 
We typically consider the scalar output (i.e., $N_L=1$) unless otherwise noted.
Under the same setup, given $\Qps$, we also define a \emph{quantized neural network} $f_{\theta,\mcI}(\,\cdot\,;\Qps):\Qps^{d}\to\Qps$ that has the following properties:

\begin{itemize}
    \item $f_{\theta,\mcI}$ has quantized weights $\bfw_{l,i}\in\Qps^{|\mcI_{l,i}|}$ and biases $b_{l,i}\in\bbQ_{\infty,s}$ for all $l\in[L]$ and $i\in[N_l]$ in all affine transformations and
    \item it has quantized outputs for all affine transformations and activation functions.
\end{itemize}
Namely, a quantized network $f_{\theta,\mcI}$ can be expressed as
\begin{align}
f_{\theta,\mcI}(\bfx;\Qps)\defeq\round{\rho_L}_{\Qps}\!\circ\round{\sigma}_{\Qps}\!\circ\round{\rho_{L-1}}_{\Qps}\!\circ\cdots\circ\round{\sigma}_{\Qps}\!\circ\round{\rho_1}_{\Qps}(\bfx)\label{eq:qnn-def},
\end{align}
where $\round{\rho_l}_{\Qps}$ and $\round{\sigma}_{\Qps}$ are quantized versions of $\rho_l$ and $\sigma$ as in \cref{eq:round-fcn}.
We note that we do not perform rounding after each of the elementary operations (e.g., addition or multiplication) in an affine transformation but perform a single rounding after computing the whole affine transformation. 
This assumption reflects practical implementations of quantization networks that typically use quantized weights/activation function values but high-precision intermediate operations \cite{jacob2018quantization,ligptaq,spinquant}.
In addition, since quantized networks often use high-precision bias parameters (i.e., $b_{l,i}$) while weights are in low-precision (i.e., $\bfw_{l,i}$) to reduce memory and multiplication costs \cite{jacob2018quantization}, we also assume high precision biases (i.e., $b_{l,i}\in\Qis$) and low-precision weights (i.e., $\bfw_{l,i}\in\Qps^{|\mcI_{l,i}|}$) in our quantized network definition \cref{eq:qnn-def}.

We call a quantized network defined in \cref{eq:qnn-def} as a ``$\sigma$ quantized network under $\Qps$''. To simplify notation, we frequently omit $\mcI$ (and $\theta$) and use $f_\theta$ (or $f$) to denote $f_{\theta,\mcI}$.
We say ``$\sigma$ quantized networks can universally represent'' if for any $d\in\bbN$ and for any $f^*:\Qps^d\to\Qps$, there exists a $\sigma$ quantized network $f$ such that $f=f^*$ on $\Qps^d$.

\tcv{
For $m\in\bbN$ and $x\in\bbR$, we use $\sum_m x$ to denote $m$ times repeated additions of $x$, i.e., $\sum_mx=mx$.
Specifically, we use $\sum_m x$ (instead of $mx$) to explicitly state that the resulting value can be implemented using $m$ additions of (the outputs of) neural networks.
Consider a quantized neural network $f:\Qps^{d_0}\to \Qps^{d_1}$ and affine transformations $\rho_1:\Qps^{d_1}\to \Qps^{d_2}$, $\rho_2:\Qps^{d_2}\to \Qps^{d_3}$ for $d_0,d_1, d_2, d_3\in\bbN$.
Define $g$ as 
\begin{equation*}
    g\defeq \rho_2\circ\round{\sigma}\circ\round{\rho_1}\circ\round{\sigma}\circ f.
\end{equation*}
Then, $\sum_m g$ can be implemented by a quantized neural network of the same depth: $\rho'_1: \Qps^{d_1}\to \Qps^{d_2m}$ is the concatenation of  $m$ copies of $\rho_1$ and $\rho'_2: \Qps^{d_2m}\to \Qps^{d_3}$ performs the summation of the $m$ concatenated inputs (i.e., the $m$ outputs of $\rho_1$ applied in parallel).
Thus, 
\begin{equation}\label{eq:mtimes_summation}
    \sum_m g = \rho'_2\circ\round{\sigma}\circ\round{\rho'_1}\circ\round{\sigma}\circ f.
\end{equation}
}

\section{Universal representation via quantized networks}\label{sec:main_results} %
In this section, we analyze the universal representation property of quantized networks. %
In \cref{sec:counterexample}, we provide a necessary condition on activation functions and $\Qps$ for universal representation. %
In \cref{sec:suff_cond}, we provide a sufficient condition on activation functions for universal approximation and show that various practical activation functions satisfy our sufficient condition (e.g.,  $\Sigmoid$, $\relu$, $\elu$, $\SoftPlus$, $\SiLU$, $\Mish$, and $\GeLU$) for any $\Qps$ with $p\ge3$.
We also show that our sufficient condition coincides with the necessary condition introduced in \cref{sec:counterexample} (i.e., our sufficient condition is necessary) for a large class of practical activation functions including $\Sigmoid$, $\relu$, $\elu$, $\SoftPlus$, $\SiLU$, $\Mish$, and $\GeLU$.
We further extend our results to quantized networks with binary weights (i.e., $\bfw_{l,i}\in\{-1,1\}^{|\mcI_{l,i}|}$) in \cref{sec:binary}.
Proofs are deferred to \cref{sec:proofs}.
Throughout this section, we use $\round{\cdot}$ to denote $\round{\cdot}_{\Qps}$ to simplify notation.

\subsection{Necessary condition for universal representation}\label{sec:counterexample}

Classical universal approximation theorems state that under real parameters and exact mathematical operations, neural networks with two layers and any non-polynomial activation function can universally approximate \cite{leshno1993multilayer}. Furthermore, for deeper networks, it is known that networks using non-affine polynomial activation functions can also universally approximate \cite{kidger2020universal}.
This implies that under real parameters and exact mathematical operations, any non-affine continuous activation function suffices for universal approximation.
However, this is not the case for the quantized networks. In this section, we show that there are non-affine activation functions and $\Qps$ for which quantized networks cannot universally represent.
In particular, we study such activation functions and $\Qps$ by formalizing a necessary condition on the activation function $\sigma$ and $\Qps$ for universal approximation.

Recall a $\sigma$ quantized network $f:\Qps^d\to\Qps$ defined in \cref{eq:qnn-def}:
\begin{align*}
f(\bfx;\Qps)=\round{\rho_L}\circ\round{\sigma}\circ\round{\rho_{L-1}}\circ\cdots\circ\round{\sigma}\circ\round{\rho_1}(\bfx).
\end{align*}
If $f$ has more than one layer, since \tcv{the output dimension of $f$ is one and}, the output of $\round{\sigma}\circ\round{\rho_{L-1}}\circ\cdots\circ\round{\sigma}\circ\round{\rho_1}$ is always in $\round{\sigma}(\Qps)$,

given the last layer weights $\bfw_{L,1}=(w_{L,1,1},\dots,w_{L,1,|\mcI_{L,1}|})\in\Qps^{|\mcI_{L,1}|}$
and bias $b_{L,1}\in\bbQ_{\infty,s}$ for $\round{\rho_L}$, the output of $f$ always satisfies 
\begin{align}
f(\bfx;\Qps)&\in\left\{\round{b_{L,1}+\sum_{j=1}^{|\mcI_{L,1}|}w_{L,1,j}z_{j}}:z_j\in\round{\sigma}(\Qps)\right\},\label{eq:nn-output}
\end{align}
for all $\bfx\in\Qps^d$.
To introduce our necessary condition, for each $b\in\Qis$, we define
\begin{align*}
\mcN_{\sigma,p,s,b}\defeq\left\{\round{b+\sum_{i=1}^nw_ix_i}:n\in\bbN_0,w_i\in\Qps,x_i\in\round{\sigma}(\Qps)~\forall i\in[n]\right\}.
\end{align*}
Then, by \cref{eq:nn-output}, one can easily observe that $f(\Qps^d;\Qps)\subset\mcN_{\sigma,p,s,b_{L,1}}$, which implies that for any $\sigma$ quantized network $\tcv{f}$ under $\Qps$ with more than one layer, we have
\begin{align}
\tcv{f}(\Qps^d;\Qps)\subset\mcN_{\sigma,p,s,b}\label{eq:nn-output2}
\end{align}
for some $b\in\Qis$.

We are now ready to introduce our necessary condition. Suppose that $\sigma$ quantized networks under $\Qps$ with more than one layer can universally represent. 
Let $f^*:\Qps^d\to\Qps$ be a target function such that $f^*(\Qps^d)=\Qps$ (e.g., $f^*(\bfx)=x_1$). 
Since $\sigma$ quantized networks can universally represent, there exists a $\sigma$ quantized network $\tcv{f}$ with more than one layer such that $\tcv{f}=f^*$. 
Furthermore, since $f^*(\Qps^d)=\Qps$, by \cref{eq:nn-output2}, we must have 
\begin{align*}
\Qps=f^*(\tcv{\Qps^d})=\tcv{f}(\tcv{\Qps^{d}};\Qps)\subset\mcN_{\sigma,p,s,b}\subset\Qps,
\end{align*}
for some $b\in\Qis$; that is, $\mcN_{\sigma,p,s,b}=\Qps$ for some $b\in\Qis$, which is our necessary condition. 
We formally state this necessary condition and extend our analysis to include one-layer networks in the following theorem.
For completeness, we provide the detailed proof of \cref{thm:necessity} in \cref{sec:pfthm:necessity}. %

\begin{theorem}\label{thm:necessity}
Let $\sigma:\bbR\to\bbR$ and $p,s\in\bbN$. If $\sigma$ quantized networks under $\Qps$ can universally represent, then there exists $b\in\Qis$ such that
\begin{align}
\mcN_{\sigma,p,s,b}=\Qps.\label{eq:necessity}
\end{align}
\end{theorem}

\cref{thm:necessity} states that if $\mcN_{\sigma,p,s,b}\ne\Qps$ for all $b\in\Qis$, then $\sigma$ quantized networks cannot universally represent. 
\tcv{Note that this necessary condition is independent of the input dimension $d$.}
Based on \cref{thm:necessity}, we characterize a class of activation functions and $\Qps$ that cannot universally represent in the following lemma. We defer the proof of \cref{lem:necessity} in  \cref{sec:pflem:necessity}.
\begin{lemma}\label{lem:necessity}
Let $\sigma:\bbR\to\bbR$ and $p,s\in\bbN$.
Suppose that there exists a natural number 
$r \in \bbN$ such that $ 3 \le r$,  $s r\in \Qps$, $\round{\sigma}(x)\in \bbZ$, and $r \mid \round{\sigma}(x)$ for all $x\in \Qps$.\footnote{$a \mid b$ denotes that $b$ is divisible by $a$ for $a,b \in \bbZ$.} If $r=3$, we further assume that $2\mid p$. Then, $\mcN_{\sigma,p,s,b}\ne\Qps$ for all $b\in\Qis$.
\end{lemma}

While any non-affine continuous activation function suffices for universal approximation under real parameters and exact mathematical operations, \cref{lem:necessity} implies that there can be non-affine $\sigma$ and $\Qps$ such that $\sigma$ quantized networks cannot universally represent. For example, the following corollary shows one such case, with $r=5$.

\begin{corollary}\label{cor:necessity}
 For $p,s\in\bbN$ such that $p \ge 2$, $5s\in \Qps$ and for $\sigma(x) = 5s\times \hardtanh(x)$, it holds that $\mcN_{\sigma,p,s,b}\ne\Qps$ for all $b\in\Qis$.
\end{corollary}

\subsection{Sufficient condition for universal representation}\label{sec:suff_cond}

In this section, we introduce a sufficient condition for universal representation of quantized networks. That is, if an activation function and $\Qps$ satisfy our sufficient condition, then we can represent any target continuous function from $\Qps^d$ to $\Qps$ via some quantized network. To this end, we explicitly construct indicator functions with some coefficient $\gamma\in\Qps$ using quantized networks: for a $d$-dimensional \emph{quantized cube} $\mcC=(\prod_{i=1}^d[\alpha_i,\beta_i])\cap\Qps^d$, 
\begin{align*}
\gamma\times\indc{\mcC}{\bfx}=\begin{cases}
\gamma~&\text{if}~\bfx\in\mcC,\\
0~&\text{if}~\bfx\notin\mcC.
\end{cases}
\end{align*}
We then represent a target function (say $f^*:\Qps^d\to\Qps$) as
\begin{align}
f^*(\bfx)=\round{\sum_{i=1}^k\gamma_i\times\indc{\mcC_i}{\bfx}} \tcv{\quad \text{for all } \bfx\in \Qps^d},\label{eq:univ_approx_indc}
\end{align}
and for some partition $\{\mcC_1,\dots,\mcC_k\}$ of $\Qps^d$,
where each $\mcC_i$ is a $d$-dimensional quantized cube of a small sidelength (e.g., $\mcC_i$ is a singleton set), and for some $\gamma_i\in\Qps$ representing $f^*(\mcC_i)$.

From \cref{eq:univ_approx_indc}, it is easy to observe that if we can implement 
 \tcv{
\begin{align*}
\mathbb{Q}_{p,s}^d &\longrightarrow \mathbb{Q}_{p,s},\\  \mathbf{x} &\longmapsto \gamma \times \indc{\mcC}{\bfx}       ,
\end{align*}}
for arbitrary $\gamma\in\Qps$ using a $\sigma$ quantized network, then $\sigma$ quantized networks can universally represent. 
However, as we observed in the necessary condition for universal representation in \cref{thm:necessity}, not all activation functions can implement every $\gamma\times\indc{\mcC}{\bfx}$ (if they could, quantized networks using any activation function would universally represent). Hence, in this section, we derive a sufficient condition on activation functions $\sigma$ and $\Qps$ under which we can implement $\gamma\times\indc{\mcC}{\bfx}$ for all $\gamma\in\Qps$.

The rest of this section is organized as follows: we first introduce a class of activation functions and $\Qps$ that we focus on; then, we characterize a class of $\gamma\in\Qps$ such that $\sigma$ quantized network can implement $\gamma\times\indc{\mcC}{\bfx}$.
Using this, we next formalize our sufficient condition. We also verify whether quantized networks under practical activation functions and quantization setups can universally represent using our sufficient condition. We lastly discuss the necessity of our sufficient condition.

\subsubsection{Activation functions of interest}
We primarily consider activation functions and $\Qps$ satisfying the following condition. 
\begin{condition}\label{cond:sufficency}
For an activation function $\sigma:\bbR\to\bbR$ and $\Qps$, there exist $\alpha,\beta\in\{-1,1\}$, $\rho:\bbR\to\bbR$, and $z\in\Qps$ satisfying the following: 
\begin{itemize}
\item $\sigma(x)=\alpha\rho(\beta x)$,
\item $z\ne\min\Qps$, %
\item $\round{\rho}(x)=\max\round{\rho}(\Qps)$ for all $x\in\Qps$ such that $x\ge z$, and
\item $\round{\rho}(x)<\max\round{\rho}(\Qps)$ for all $x\in\Qps$ such that $x<z$.
\end{itemize}
\end{condition}
\cref{cond:sufficency} with the case $\alpha=\beta=1$ characterizes a class of activation functions and $\Qps$ such that 
\begin{itemize}
\item $\round{\sigma}$ is non-constant on $\Qps$ and
\item $\round{\sigma}(x)=\max_{z\in\Qps}\round{\sigma}(z)$ if and only if $x\in[z,\max\Qps]\cap\Qps$.
\end{itemize}
Compared to the $\alpha=\beta=1$ case, different values of $\alpha$ and $\beta$ change the \emph{maximum} to the \emph{minimum} and the interval $[z,\max\Qps]\cap\Qps$ to the interval $[\min\Qps,z]\cap\Qps$, respectively.

It is easy to observe that any monotone activation function $\sigma$ with non-constant $\round{\sigma}$ satisfies \cref{cond:sufficency}, e.g., $\relu$, leaky-$\relu$, $\SoftPlus$, $\Sigmoid$, etc.
We formally present this observation in \cref{lem:monotone}. The proof is provided in \cref{sec:pflem:monotone}.

\begin{lemma}\label{lem:monotone}
Let $p,s\in\bbN$. If $\sigma:\bbR\to\bbR$ is monotone and there exist $x,y\in\Qps$ such that $\round{\sigma}(x)\ne\round{\sigma}(y)$ (i.e., $\round{\sigma}$ is non-constant on $\Qps$), then $\sigma$ and $\Qps$ satisfy \cref{cond:sufficency}.
\end{lemma}
We note that popular non-monotone activation functions such as GELU, SiLU, and Mish also satisfy \cref{cond:sufficency} for all $\Qps$. 

\subsubsection{Implementing indicator functions via quantized networks}
We implement indicator functions using $\sigma$ quantized networks under $\Qps$ for $\sigma$ and $\Qps$ satisfying \cref{cond:sufficency}.
To describe a class of implementable indicator functions, we define the following sets: for $b\in\Qis$,
\tcv{
\begin{align*}
\mathcal V_{\sigma,p,s}&\defeq\left\{\round{\sigma}(x)-\round{\sigma}(y):x,y\in\Qps\right\},\\
\mcS_{\sigma,p,s}^\circ&\defeq\left\{\sum_{i=1}^nw_ix_i:n\in\bbN_0,w_i\in\Qps,x_i\in\mcV_{\sigma,p,s}~\forall i\in[n]\right\}.
\end{align*}}
$\mcV_{\sigma,p,s}$ is a set of all gaps between possible outputs of $\round{\sigma}$, and $\mcS_{\sigma,p,s}^\circ$ is a set of all linear combinations of elements in $\mcV_{\sigma,p,s}$.
Using this definition, we characterize a class of indicator functions (of the form $\gamma\times\indc{\mcC}{\bfx}$) via the following lemma.
We provide the proof of \cref{lem:indicator} in \cref{sec:pflem:indicator}. 
\begin{lemma}\label{lem:indicator}
Let $\sigma:\bbR\to\bbR$ and $p,s,d\in\bbN$ satisfying $s\le 2^p-1$.
Let $\alpha_1,\beta_1,\dots,\alpha_d,\beta_d\in\Qps$ such that $\alpha_i\le\beta_i$ for all $i\in[d]$ and let $\mcC=(\prod_{i=1}^d[\alpha_i,\beta_i])\cap\Qps^d$.
Suppose that $\sigma$ and $\Qps$ satisfy \cref{cond:sufficency}.
Then, for each $\gamma\in\mcS_{\sigma,p,s}^\circ$, there exist $d'\in\bbN$, an affine transformation $\rho:\bbR^{d'}\to\bbR$ with quantized weights and without bias (i.e., $\rho=\aff(\,\cdot\,;\bfw,0,\mcI)$ for some $\bfw\in\Qps^{|\mcI|}$ and $b=0\in\Qis$), and a three-layer $\sigma$ quantized network $f(\,\cdot\,;\Qps):\Qps^d\to\Qps^{d'}$ such that 
\begin{align}
    \rho\circ\round{\sigma}\circ f(\bfx;\Qps)= \gamma \times\indc{\mcC}{\bfx}\quad\forall\bfx\in\Qps.\label{eq:lem:indicator}
\end{align}
\end{lemma}

\cref{lem:indicator} states that if $\gamma\in\mcS_{\sigma,p,s}^\circ$, then $\gamma\times\indc{\mcC}{\bfx}$ can be implemented by a composition of a $\sigma$ quantized network, a quantized activation function, and a linear transformation $\rho$ with quantized weights and \emph{unquantized output}.
Here, we do not introduce the bias parameter in $\rho$ and do not quantize the output of $\rho$; this is because our final 
quantized network construction (say $f$) that represents a target function has the following form: for a partition $\mcC_1,\dots,\mcC_k$ of $\Qps$,
\begin{align}
f(\bfx;\Qps)=\round{b+\sum_{i=1}^k\gamma_i\times\indc{\mcC_i}{\bfx}}=\round{\sum_{i=1}^k(\gamma_i+b)\times\indc{\mcC_i}{\bfx}},
\label{eq:univ_approx_indc2}
\end{align}
as in \cref{eq:univ_approx_indc}. Note that the second equality holds since $\mcC_1,\dots,\mcC_k$ is a partition of $\Qps$.
Namely, we will add the bias parameter and quantize the final output after summing the indicator functions.

\subsubsection{Our sufficient condition}
To describe our sufficient condition for universal representation, we define
\begin{align*}
\mcS_{\sigma,p,s,b}&\defeq \left\{\round{z+b}\in\Qps:z\in\mcS_{\sigma,p,s}^\circ\right\}\\
&=\left\{\round{b+\sum_{i=1}^nw_ix_i}\in\Qps:n\in\bbN_0,w_i\in\Qps,x_i\in\mcV_{\sigma,p,s}~\forall i\in[n]\right\}.
\end{align*}
Given \cref{lem:indicator} and \cref{eq:univ_approx_indc2}, 
if $\{\mcC_1,\dots,\mcC_k\}$ is a partition of $\Qps^d$ where each $\mcC_i$ is a quantized cube, then we can construct a $\sigma$ quantized network $f$ of the following form: for any $\gamma_1,\dots,\gamma_k\in\mcS_{\sigma,p,s,b}$, %
\begin{align}
    f(\bfx;\Qps)&=\round{\sum_{i=1}^k\gamma_i\times\indc{\mcC_i}{\bfx}}=\sum_{i=1}^k\round{\gamma_i}\times\indc{\mcC_i}{\bfx}.\label{eq:univ_approx_indc3}
\end{align}
Here, the last inequality holds since $\mcC_i\cap\mcC_j=\emptyset$ if $i\ne j$.
Since $\round{\gamma_i}\in\mcS_{\sigma,p,s,b}$ if $\gamma_i-b\in\mcS_{\sigma,p,s}^\circ$, one can conclude that we can construct a $\sigma$ quantized network $f$ as in \cref{eq:univ_approx_indc3} using \cref{lem:indicator}, for any $\round{\gamma_i}\in\mcS_{\sigma,p,s,b}$.
Namely, if $\mcS_{\sigma,p,s,b}=\Qps$ for some $b\in\Qis$, then $\sigma$ quantized networks can universally represent by choosing proper $\{\mcC_1,\dots,\mcC_k\}$ (e.g., each $\mcC_i$ is a singleton set).
We formally present this sufficient condition in \cref{thm:sufficiency}, whose proof is provided in \cref{sec:pfthm:sufficiency}.

\begin{theorem}\label{thm:sufficiency}
Let $\sigma:\bbR\to\bbR$ and $p,s\in\bbN$ satisfying $s\le 2^p-1$. Suppose that $\sigma:\bbR\to\bbR$ and $\Qps$ satisfy \cref{cond:sufficency}. 
If there exists $b\in\Qis$ such that 
\begin{align*}
\mcS_{\sigma,p,s,b}=\Qps,
\end{align*} 
then $\sigma$ quantized networks under $\Qps$ can universally represent.
\end{theorem}
One representative activation function that satisfies the condition in \cref{thm:sufficiency} for all $p,s\in\bbN$ is the identity function $\sigma(x)=x$; in this case, $\mcS_{\sigma,p,s,0}=\Qps$.
This shows a gap between classical universal approximation results and ours; if a network uses real parameters and exact mathematical operations, then it can only express affine maps with the identity activation function, i.e., the network cannot universally approximate. 
Nevertheless, quantized networks with the identity activation function can universally represent as stated in \cref{thm:sufficiency}. This is because fixed-point additions and multiplications in quantized networks are non-affine due to rounding errors.
We note that a similar observation has been recently made under the floating-point arithmetic \cite{hwang2025floating, hwang2025floatingpoint}.

We next provide an easily verifiable condition for activation functions that guarantees $\mcS_{\sigma,p,s,b}=\Qps$ (i.e., the universal representation property).
We provide the proof of \cref{lem:sufficiency}  in \cref{sec:pflem:sufficiency}. 

\begin{lemma}\label{lem:sufficiency}
Let $\sigma:\bbR\rightarrow\bbR$ be a continuous function and natural numbers $p,s\in\bbN$ satisfying $s\le 2^p-1$.
If $\Qps$ and $\sigma$ satisfy one of the following conditions, then $\mcS_{\sigma,p,s,b}=\Qps$.
    \begin{enumerate}[leftmargin=0.4in]
        \item[(A1)] There exist $q_1, q_2\in \bbZ$ such that, ${-2^p+1} \leq q_1 < q_2\leq {2^p-1}$, $\sigma$ is differentiable on $\left(\frac{q_1}{s}, \frac{q_2}{s} \right)$, $|\sigma'(x)|< 1$, and $ |\sigma(x)| \leq \frac{2^p-1}{s}$ for $x \in \left(\frac{q_1}{s}, \frac{q_2}{s}\right)$, and
        \begin{equation*}
            \left|\sigma\left(\frac{q_2}{s}\right) - \sigma\left(\frac{q_1}{s}\right)\right|\geq \frac{1}{s}.
        \end{equation*}
        \item[(A2)] There exist $q_1, q_2\in \bbZ$ such that ${-2^p+1} \leq q_1 < q_2\leq {2^p-1}$, $\sigma$ is differentiable on $\left(\frac{q_1}{s}, \frac{q_2}{s}\right)$, $|\sigma'(x)|\leq 1$, and $0\leq \sigma(x)\leq \frac{2^p-1}{s}$ for $x \in \left( \frac{q_1}{s} ,\frac{q_2}{s}\right)$, and
        \begin{equation*}
            \left|\sigma\left(\frac{q_2}{s}\right) - \sigma\left(\frac{q_1}{s}\right)\right|\geq \frac{1}{s}.
        \end{equation*}
        \item[(A3)] There exist $q_1, q_2\in \bbZ$ such that ${-2^p+1} \leq q_1 < q_2\leq {2^p-1}$, $\sigma$ is differentiable on $\left(\frac{q_1}{s}, \frac{q_2}{s} \right)$, $1\leq \sigma'(x)\leq 2$, and  $|\sigma(x)| \leq \frac{2^p-1}{s}$ for  $ x \in \left( \frac{q_1}{s} , \frac{q_2}{s} \right)$, and
        \begin{equation*}
             \left|\sigma\left(\frac{q_2}{s}\right) - \sigma\left(\frac{q_1}{s}\right)\right| < \frac{2(q_2 - q_1) - 1}{s}.
        \end{equation*} 
    \end{enumerate}
\end{lemma}

Using \cref{lem:sufficiency}, we show that for many practical activation functions such that $\relu$, $\ELU$, $\SiLU$, $\Mish$, $\GeLU$, $\SoftPlus$, and $\Sigmoid$, we have $\Qps =  \mcS_{\sigma,p,s,b}$.
We provide the proof of \cref{lem:activation}  in \cref{sec:pflem:activation}.
\begin{lemma}\label{lem:activation}
Consider $p,s\in \bbN$ satisfying $s\le 2^p-1$.
For the activation functions $\relu$ and $\ELU$, $\Qps =  \mcS_{\sigma,p,s,b}$.
For $\SiLU$, $\Mish$, $\GeLU$, and $\Sigmoid$, if $p\ge3$, then $\Qps =  \mcS_{\sigma,p,s,b}$. For $\SoftPlus$, if $p \ge 4$ or $p=3$ and $ 1 \le s \le 5$, then $\Qps =  \mcS_{\sigma,p,s,b}$.
\end{lemma}
\tcv{In \cref{lem:activation}, we present conditions on $p$ and $s$ for various activation functions to universally represent. Specifically, the condition on $s$ is only to ensure $1\in\Qps$, which is sufficient for universal representation (see \cref{sec:binary} for more details).
In practice, researchers typically share a single $p$ for all layers but use different values of $s$ for different layers to optimize the quantization performance; our problem setup sharing a single $s$ for all layers can be viewed as a special case of this, i.e., our sufficiency results also extend to a practical setup if we can choose $s$ in each layer by ourselves.
For example, a conventional way to choose $s$ to cover the range of (full-precision) weights in an unquantized model. In this case, for the INT8 quantization ($p=7$) using PyTorch’s post-training static quantization framework, we found that $s$ ranges from 
$125$ to $480$ in the ResNet18 \cite{he2016deep}, 
$38$ to $351$ in the MobileNetV2  \cite{sandler2018mobilenetv2} pretrained on ImageNet \cite{deng2009imagenet}, and $12$ to $638$  in the DistilBERT \cite{sanh2019distilbert}
pretrained on BookCorpus \cite{bandy2021addressing}.

One notable empirical observation here is that while $1\in\Qps$ in the first layers, $\Qps$ may not contain $1$ from the second layer. Nevertheless, we can still simulate the weight of value one (or larger) by duplicating the same neurons, multiplying them the same weight (e.g., $1/s$) and adding them (e.g., $s$ times). %
We note that such a duplication strategy may not work in the first layer as we typically do not duplicate the inputs.}

\subsubsection{Our necessary and sufficient condition}
\cref{thm:sufficiency} states that $\mcS_{\sigma,p,s,b}=\Qps$ for some $b\in\Qis$ is \emph{sufficient} for universal representation, while \cref{thm:necessity} states that $\mcN_{\sigma,p,s,b}=\Qps$ for some $b\in\Qis$ is \emph{necessary}.
Hence, by combining these two results, one can observe that if $\mcN_{\sigma,p,s,b}=\mcS_{\sigma,p,s,b}$ for all $b$, then $\mcS_{\sigma,p,s,b}=\Qps$ for some $b$ is \emph{necessary and sufficient} for universal representation.
\begin{corollary}\label{cor:nece_suff}
Let $\sigma:\bbR\to\bbR$ and $p,s\in\bbN$ satisfying $s\le 2^p-1$. Suppose that $\sigma:\bbR\to\bbR$ and $\Qps$ satisfy \cref{cond:sufficency}.
If $\mcN_{\sigma,p,s,b}=\mcS_{\sigma,p,s,b}$ for all $b\in\Qis$, then $\sigma$ quantized networks can universally represent if and only if
there exists $b\in\Qis$ such that $\mcS_{\sigma,p,s,b}=\Qps$.
\end{corollary}

Then, when do we have $\mcN_{\sigma,p,s,b}=\mcS_{\sigma,p,s,b}$ for all $b\in\Qis$? To answer this question, we provide conditions on the activation function $\sigma$ and $\Qps$ that guaranty $\mcN_{\sigma,p,s,b}=\mcS_{\sigma,p,s,b}$ for all $b\in\Qis$.
We provide the proof of \cref{lem:nece_suff} in \cref{sec:pflem:nece_suff}.

\begin{lemma}\label{lem:nece_suff}
Let $\sigma:\bbR\to\bbR$ and $p,s\in\bbN$ satisfying $s\le 2^p-1$.
If there exists $x\in\Qps$ such that $\round{\sigma}(x)\in\mcV_{\sigma,p,s}$, then, $\mcN_{\sigma,p,s,b}=\mcS_{\sigma,p,s,b}$ for all $b\in\Qis$.
In particular, if there exists $x\in\Qps$ such that $\round{\sigma}(x)=0$, then, $\mcN_{\sigma,p,s,b}=\mcS_{\sigma,p,s,b}$ for all $b\in\Qis$.
\end{lemma}
One can easily verify that most practical activation functions whose value at zero is zero naturally satisfy the assumption of \cref{lem:nece_suff}.

\subsection{Universal representation with binary weights}\label{sec:binary}
Recent research has demonstrated that neural networks with parameters quantized to one-bit precision can achieve performance comparable to their full-precision counterparts while significantly reducing multiplication costs \cite{ma2024era, wang2023bitnet}.
In this section, we present both a necessary condition and a sufficient condition for networks with binary weights to possess universal representation property as in \cref{sec:counterexample,sec:suff_cond}. %

Recall a $\sigma$ quantized network $f:\Qps^d\to\Qps$ defined as in \cref{eq:qnn-def}:
\begin{align}
f=\round{\rho_L}\circ\round{\sigma}\circ\round{\rho_{L-1}}\circ\cdots\circ\round{\sigma}\circ\round{\rho_1}.\label{eq:b_qnn-def}
\end{align}
We say that ``$f$ has binary weights'' if all its weights in the affine transformations $\rho_l$ are binary.
Specifically, for each $l,i$,
\begin{equation*}
    \bfw_{l,i}\in\{-1,1\}^{|\mcI_{l,i}|}.
\end{equation*}
Here, we allow non-binary bias parameters, i.e., $b_{l,i}\in\Qis$ as in \cite{ma2024era, wang2023bitnet}.

As $\mcN_{\sigma,p,s,b}, \mcS_{\sigma,p,s}^\circ$, and $\mcS_{\sigma,p,s,b}$ for quantized networks (see \cref{sec:counterexample,sec:suff_cond}), we analogously define the sets $\mcB\mcN_{\sigma,p,s,b}$, $\mcB\mcS_{\sigma,p,s}^\circ$, and $\mcB\mcS_{\sigma,p,s,b}$ for quantized networks with binary weights as follows:
\begin{align*}
\mcB\mcN_{\sigma,p,s,b}&\defeq\left\{\round{b+\sum_{i=1}^nw_ix_i}:n\in\bbN_0, w_i\in \{-1, 1\},x_i\in\round{\sigma}(\Qps)~\forall i\in[n]\right\},\\
\mcB\mcS_{\sigma,p,s}^\circ&\defeq\left\{\sum_{i=1}^nw_ix_i:n\in\bbN_0, w_i\in\{-1, 1\},x_i\in\mcV_{\sigma,p,s}~\forall i\in[n]\right\},\\
\mcB\mcS_{\sigma,p,s,b}&\defeq\left\{\round{b+z}:z\in\mcB\mcS_{\sigma,p,s}^{\circ}\right\}\\
&=\left\{\round{b+\sum_{i=1}^nw_ix_i}:n\in\bbN_0, w_i\in\{-1,1\},x_i\in\mcV_{\sigma,p,s}~\forall i\in[n]\right\}.
\end{align*}
Using these definitions, we derive a necessary condition and a sufficient condition for quantized networks with binary weights to achieve universal representation via the following lemmas and theorems. We present their proofs in Sections~\ref{sec:pfthm:b_necessity}--\ref{sec:pflem:b_sufficiency}, which are analogous to the results for general quantized networks.

We first present \cref{thm:b_necessity}, \cref{lem:b_necessity}, and \cref{cor:b_necessity} which describe our necessary condition. 
\begin{theorem}\label{thm:b_necessity}
Let $\sigma:\bbR\to\bbR$ and $p,s\in\bbN$. If $\sigma$ quantized networks under $\Qps$ with binary weights can universally represent, then there exists $b\in\Qis$ such that
\begin{align}
\mcB\mcN_{\sigma,p,s,b}=\Qps.\label{eq:b_necessity}
\end{align}
\end{theorem}

\begin{lemma}\label{lem:b_necessity}
Let $\sigma:\bbR\to\bbR$ and $p,s\in\bbN$. 
Suppose that there exists a natural number $ r\in \mathbb{N}$ such that $3 \le r $, $s  r\in \Qps$ and $r\mid s \round{\sigma}(x)$ for all $x\in \Qps$. 
Further assume $2\mid p$ if $r =3$. 
Then, $\mcB\mcN_{\sigma,p,s,b}\ne\Qps$ for all $b\in\Qis$.
\end{lemma}
Note that unlike \cref{lem:necessity} which requires the conditions $\round{\sigma}(x)\in\bbZ$ and $r\mid \round{\sigma}(x)$, \cref{lem:b_necessity} only requires $r\mid s\round{\sigma}(x)$.
This is because quantized networks with binary weights have less expressivity compared to general quantized networks with possibly non-binary weights.
In other words, it is easier to find functions that cannot be represented by quantized networks with binary weights, compared to the non-binary weight case. 
Using \cref{lem:b_necessity}, we can also show that quantized networks with binary weights and the $5\times\hardtanh(x)$ activation function may not universally represent. 
\begin{corollary}\label{cor:b_necessity}
 For $p,s\in\bbN$ such that $5\in \Qps$ and for the activation function $\sigma(x) = 5\times \hardtanh(x)$, it holds that $\mcB\mcN_{\sigma,p,s,b}\ne\Qps$ for all $b\in\Qis$.
\end{corollary}

We now present our sufficient condition for quantized networks with binary weights to achieve universal representation via \cref{lem:b_indicator}, \cref{thm:b_sufficiency}, and \cref{lem:b_sufficiency}.

\begin{lemma}\label{lem:b_indicator}
Let $\sigma:\bbR\to\bbR$ and $p,s,d\in\bbN$ satisfying $s\le 2^p-1$.
Let $\alpha_1,\beta_1,\dots,\alpha_d,\beta_d\in\Qps$ such that $\alpha_i\le\beta_i$ for all $i\in[d]$ and let $\mcC=(\prod_{i=1}^d[\alpha_i,\beta_i])\cap\Qps^d$.
Suppose that $\sigma$ and $\Qps$ satisfy \cref{cond:sufficency}.
Then, for each  $\gamma\in\mcB\mcS_{\sigma,p,s}^\circ$, there exist $d'$, an affine transformation $\rho:\bbR^{d'}\to\bbR$ with binary weights and without bias, and a two-layer $\sigma$ quantized network $f(\,\cdot\,;\Qps):\Qps^d\to\Qps^{d'}$ with binary weights such that 
\begin{align*}
    \rho\circ\round{\sigma}\circ f(\bfx;\Qps)=\gamma\times\indc{\mcC}{\bfx}\quad\forall\bfx\in\Qps.
\end{align*}
\end{lemma}

\begin{theorem}\label{thm:b_sufficiency}
Let $\sigma:\bbR\to\bbR$ and $p,s\in\bbN$ satisfying $s\le 2^p-1$. Suppose that $\sigma:\bbR\to\bbR$ and $\Qps$ satisfy \cref{cond:sufficency}. 
If there exists $b\in\Qis$ such that 
\begin{align*}
\mcB\mcS_{\sigma,p,s,b}=\Qps,
\end{align*} 
then $\sigma$ quantized networks under $\Qps$ can universally represent.
\end{theorem}

\begin{lemma}\label{lem:b_sufficiency}
   Consider an activation function $\sigma:\bbR\rightarrow\bbR$, $p,s\in\bbN$ satisfying $s\le 2^p-1$, and $\Qps$ which satisfy one of the conditions in \cref{lem:sufficiency}.
    Then, $\mcB\mcS_{\sigma,p,s,b}=\Qps$.
\end{lemma}
Note that the assumptions in \cref{lem:b_sufficiency} are identical to those in \cref{lem:sufficiency}; thus, all activation functions listed in the discussion of \cref{lem:sufficiency} also satisfy the assumption of \cref{lem:b_sufficiency}.
Although the expressive power of quantized networks with binary weights is constrained due to their binary nature, most activation functions are capable of universal representation by \cref{lem:b_sufficiency}. 

Lastly, in \cref{cor:b_nece_suff} and \cref{lem:b_nece_suff}, we provide a necessary and sufficient condition for universal representation, and suggest a mild condition for our necessary and sufficient condition to be satisfied.
\begin{corollary}\label{cor:b_nece_suff}
Let $\sigma:\bbR\to\bbR$ and $p,s\in\bbN$ satisfying $s\le 2^p-1$. Suppose that $\sigma:\bbR\to\bbR$ and $\Qps$ satisfy \cref{cond:sufficency}.
If $\mcB\mcN_{\sigma,p,s,b}=\mcB\mcS_{\sigma,p,s,b}$ for all $b\in\Qis$, then $\sigma$ quantized networks can universally represent if and only if
there exists $b\in\Qis$ such that $\mcB\mcS_{\sigma,p,s,b}=\Qps$.
\end{corollary}

\begin{lemma}\label{lem:b_nece_suff}
Let $\sigma:\bbR\to\bbR$ and $p,s\in\bbN$ satisfying $s\le 2^p-1$.
If there exists $x\in\Qps$ such that $\round{\sigma}(x)\in\mcV_{\sigma,p,s}$, then $\mcB\mcN_{\sigma,p,s,b}=\mcB\mcS_{\sigma,p,s,b}$ for all $b\in\Qis$.
In particular, if there exists $x\in\Qps$ such that $\round{\sigma}(x)=0$, then $\mcB\mcN_{\sigma,p,s,b}=\mcB\mcS_{\sigma,p,s,b}$ for all $b\in\Qis$.
\end{lemma}

\section{Discussions}\label{sec:discussions}

\subsection{On na\"ive quantization of networks using real parameters}
\tcv{In this section, we show that na\"ive quantization of networks using real parameters can incur large errors. Namely, universal representation property of quantized networks does not directly follow from existing universal approximation results for networks using real parameters.
For $p,s\in\bbN$ with $s\le 2^{p}-1$, we consider a two-layer network $f:\bbR^{4s+1}\to\bbR$ %
defined as follows: 
\begin{align*}
    f(\bfx) = &\sum_2 \left( \relu\left( \sum_{i=1}^{2s+1} w_{1,i} x_i \right) \right) +  \relu\left(\sum_{i=1}^{4s+1} w_{2,i} x_i  \right)\\
    \qquad + &\sum_3 \left( -1 \times \relu\left( \sum_{i=1}^{2s+1} w_{3,i} x_i \right) \right) + \sum_2 \left(-1 \times \relu\left(  \sum_{i=1}^{s+1} w_{4,i} x_i  \right) \right)
\end{align*}
where
\begin{align*}
    (w_{1,1},w_{1,2}, w_{1,3}, \dots , w_{1,2s+1}) &= \left(1,-\frac{1}{4s} , \dots , -\frac{1}{4s}\right),\\
    (w_{2,1},w_{2,2}, w_{2,3}, \dots , w_{2,4s+1}) &= \left(-1,\frac{1}{4s} , \dots , \frac{1}{4s}\right),\\
    (w_{3,1},w_{3,2}, w_{3,3}, \dots , w_{3,2s+1}) &= \left(-1,\frac{1}{2s} , \dots , \frac{1}{2s}\right),\\
    (w_{4,1},w_{4,2}, w_{4,3}, \dots , w_{4,s+1}) &= \left(-1,\frac{1}{2s} , \dots , \frac{1}{2s}\right).
\end{align*}
Then, we have 
\begin{align*}
    f(-\bold{1})=-1, \quad f(\bold{1} )=1,
\end{align*}
where $\bold{1}=(1,\dots,1)\in\bbR^{4s+1}$.
However, after performing quantization of weights into $\Qps$, we have 
\begin{align*}
(\round{w_{1,1}},\round{w_{1,2}}, \round{w_{1,3}}, \dots , \round{w_{1,2s+1}}) &= (1,0,0, \dots , 0),\\ 
(\round{w_{2,1}},\round{w_{2,2}}, \round{w_{2,3}}, \dots , \round{w_{2,4s+1}}) &= (-1,0,0, \dots , 0),\\ 
(\round{w_{3,1}},\round{w_{3,2}}, \round{w_{3,3}}, \dots , \round{w_{3,2s+1}}) &= \left(-1,\frac{1}{s},\frac{1}{s}, \dots ,\frac{1}{s}\right),\\ 
(\round{w_{4,1}},\round{w_{4,2}}, \round{w_{4,3}}, \dots , \round{w_{4,s+1}}) &= \left(-1,\frac{1}{s},\frac{1}{s}, \dots ,\frac{1}{s}\right) ,
\end{align*}
where $\round{\cdot}$ denotes $\round{\cdot}_{\bbQ_{p,s}}$.
Therefore, we have
\begin{align*}
    \round{f}(-\bold{1})=1, \quad \round{f}(\bold{1})=-1.
\end{align*}
This implies that a na\"ive quantization of a network using real parameters can incur large errors, and hence, existing universal approximation theorems for real parameters do not directly extend to quantized networks.
}
\subsection{Number of parameters in our universal representator}
In this section, we quantitatively analyze the number of parameters in our universal representator in \cref{thm:sufficiency}. 
Under the assumption in \cref{thm:sufficiency} (i.e., $\mcS_{\sigma,p,s,b}=\Qps$ for some $b\in\Qis$), we represent a target function $f^*:\Qps^d\to\Qps$ by a quantized network $f$ as
\begin{align*}
f(\bfx;\Qps)=\sum_{\bfz\in\Qps^d}f^*(\bfz)\times \indc{\{\bfz\}}{\bfx}=b+\sum_{\bfz\in\Qps^d}(f^*(\bfz)-b)\times \indc{\{\bfz\}}{\bfx}.
\end{align*}
Hence, if we count the number of layers and parameters in our indicator function construction, then we can compute the number of layers and parameters in our universal representator.
Before counting the number of parameters in our universal representator, we first introduce the following lemma. The proof of \cref{lemma:bchange} is in \cref{sec:pflem:bchange}.

\begin{lemma}\label{lemma:bchange}
Let $p,s\in\bbN$ satisfying $s\le 2^p-1$ and $p \ge  3$. Suppose $ \mcS_{\sigma,p,s,b} = \Qps$ for some $b\in\Qis$. 
Then, there exists $b' \in \Qis$ such that $|b'| \le q_{\max}$ and 
$$\left\{ \round{\gamma + b'} : \gamma \in \mcS_{\sigma,p,s}^\circ, \;  |\gamma| \le 2q_{\max}+\frac3{2s}\right\}=\Qps.$$
\end{lemma}
\cref{lemma:bchange} implies that if $ \mcS_{\sigma,p,s,b} = \Qps$ for some $b\in\Qis$, then there exists $b'\in\Qis$ such that $ \mcS_{\sigma,p,s,b'} = \Qps$.
Furthermore, it also shows that for any $x\in\Qps$, there exists $\gamma\in\mcS^\circ_{\sigma,p,s}$ such that $|\gamma|\le2q_{\max}+\frac3{2s}$ and $x=\round{\gamma+b'}$.
Under these observations, we now count the number of parameters in our indicator function construction in the following lemma. We provide the proof of \cref{lem:param_count}  in \cref{sec:pflem:param_count}. 
\begin{lemma}\label{lem:param_count}
Let $\sigma:\bbR\to\bbR$, $p,s\in\bbN$ satisfying $s\le 2^p-1$, and $b\in\Qis$. Then, The number of parameters in $\rho\circ\round{\sigma}\circ f = \gamma \times \indc{\{\bfz\}}{\bfx}$ in \cref{lem:indicator} is upper bounded by $O\left(2^{2p}s^2d\right)$ for all $\gamma\in\mcS_{\sigma,p,s}^\circ$ such that $|\gamma| \le  2 q_{\max} +\frac{3}{2s}$.
\end{lemma}
Together with \cref{lemma:bchange}, \cref{lem:param_count} shows that when $\mcS_{\sigma,p,s,b}$ for some $b\in\Qis$,
each indicator function can be constructed by a four-layer $\sigma$ quantized network of at most $O(2^{2p}s^2d)$ parameters (see \cref{lem:indicator}).
Since $|\Qps^d| = O(2^{(p+1)d})$,
one can observe that our universal representator has four layers and at most $O(2^{2p+(p+1)d}s^2d)$ parameters.

\subsection{Approximating continuous functions using quantized networks}
In this section, we discuss approximating continuous functions using quantized networks and quantitatively analyze the number of parameters in our approximator. 
Specifically, given a target continuous function $f^*:[-q_{\max},q_{\max}]^d\to\bbR$ and an error bound $\varepsilon\ge0$, we want to approximate $f^*$ using a quantized network $f:\Qps^d\to\Qps$ as follows: for each $\bfx\in\Qps$,
\begin{align*}
|f(\bfx;\Qps)-f^*(\bfx)|\le\left|f^*(\bfx)-\round{f^*(\bfx)}\right|+\varepsilon.
\end{align*}
Here, the term $\left|f^*(\bfx)-\round{f^*(\bfx)}\right|$ is an intrinsic error that arises from the discreteness of $\Qps$, i.e., $|f(\bfx;\Qps)-f^*(\bfx)|$ can never be smaller than $|f^*(\bfx)-\round{f^*(\bfx)}|$.

If quantized networks can universally represent, then they can also represent $\round{f^*}$ (i.e., $\varepsilon=0$ is achievable). In this case, as we discussed in the previous section, $O(2^{2p+(p+1)d}s^2d)$ parameters are sufficient. 
Hence, we consider the following question: if $\varepsilon>0$, can we approximate $f^*$ using a smaller number of parameters?
To answer this question, we define the \emph{modulus of continuity} $\omega_{f^*}:\bbR_{\ge0}\to\bbR_{\ge0}$ of a continuous function ${f^*}:\mcX\to\bbR$ as
\begin{align}
    \omega_{f^*}(\delta)\defeq\sup_{\bfx,\bfx'\in\mathcal X:\|\bfx-\bfx'\|_\infty\le \delta}|{f^*}(\bfx)-{f^*}(\bfx')|,\label{eq:modulus1}
\end{align}
and we define its inverse $\omega_{f^*}^{-1}:\bbR_{\ge0}\to\bbR_{\ge0} \cup \{\infty\}$ as
\begin{align}
    \omega_{f^*}^{-1}(\varepsilon)\defeq\sup\{\delta\ge0:\omega_{f^*}(\delta)\le\varepsilon\}.\label{eq:modulus2}
\end{align}
Using $\omega_{f^*}$, we provide the number of parameters that is sufficient to approximate the target continuous function $f^*$ within $\varepsilon$ error in the following theorem.
The proof of \cref{thm:param_count} is presented in \cref{sec:pfthm:param_count}.

\begin{theorem}\label{thm:param_count}
 Let $\sigma:\bbR\to\bbR$, $p,s,d\in\bbN$ satisfying $s\le 2^p-1$, $b\in\Qis$, and $\mathcal{X}=[-q_{\max}, q_{\max}]$.
Suppose that $\sigma$ and $\Qps$ satisfy \cref{cond:sufficency}, and $\mcS_{\sigma,p,s,b}=\Qps$.
Then, for any continuous $f^*:\mathcal{X}^d\to\bbR$ with modulus of continuity $\omega_{f^*}$ and for any $\varepsilon>0$, there exists a four-layer $\sigma$ quantized network $f(\,\cdot\,;\Qps):\Qps^d\to\Qps$ of at most P parameters such that
\begin{align}
\left|f(\bfx;\Qps)-f^*(\bfx)\right|\le \left|f^*(\bfx)-\round{f^*(\bfx)}\right|+\varepsilon\label{eq:thm:param_count}
\end{align}
for all $\bfx\in\Qps^d$ where
\begin{align*}
    P = \begin{cases}
    1 \quad &\text{if} \quad  
      \omega_{f^*}^{-1}(\varepsilon) \ge 2q_{\max} ,\\
        O\left( 2^{2p}s^2 d (6 q_{\max})^d (\omega_{f^*}^{-1}(\varepsilon) )^{-d} \right) \quad &\text{if} \quad  
        \frac{1}{s} < \omega_{f^*}^{-1}(\varepsilon) < 2q_{\max} , \\
        O\left(2^{d(p+1)+2p}s^2d \right)\quad &\text{if} \quad \omega_{f^*}^{-1}(\varepsilon) \le \frac{1}{s}. 
    \end{cases}
\end{align*}
\end{theorem}

As we described in \cref{eq:univ_approx_indc}, our universal approximator is a sum of indicator functions over quantized cubes that form a partition of $\Qps^d$.
Specifically, we choose quantized cubes so that their sidelengths are at most $\omega_{f^*}^{-1}(\varepsilon)$; then $\Theta((2q_{\max})^d(\omega_{f^*}^{-1}(\varepsilon))^{-d})$ quantized cubes are sufficient for partitioning $\Qps^d$.
Furthermore, for each such quantized cube, our approximator incurs at most $\left|f^*(\bfx)-\round{f^*(\bfx)}\right|+\varepsilon$ error for all $\bfx$ in that cube by the definition of the modulus of continuity (\cref{eq:modulus1,eq:modulus2}).
Here, we use $O(2^{2p}s^2d)$ parameters for each indicator function (see \cref{lem:param_count}). Hence, our universal approximator uses $ O(2^{2p}s^3d (6q_{\max})^d (\omega_{f^*}^{-1}(\varepsilon))^{-d} )$ parameters to achieve \cref{eq:thm:param_count} in \cref{thm:param_count}.

\subsection{Comparison with floating-point results}

\tcv{A few recent works study the representability of floating-point networks whose input, parameters, and output follow floating-point format. Park et al.\ \cite{park2024expressive} show that ReLU networks and Step\footnote{A binary step function $\indc{\{x\in\bbR:x\ge0\}}x$.} networks can represent almost all floating-point functions (from/to floating-point vectors).
This result has been extended to general activation functions \cite{hwang2025floatingpoint} and interval arithmetic setup \cite{hwang2025floating}.
One major difference from our setup is that they apply rounding for each addition and multiplication in affine transformations. 
Hence, these floating-point results can be applied to standard implementations of neural networks, since the floating-point arithmetic is typically a default choice in popular neural network libraries (e.g., using PyTorch or Tensorflow).
On the other hand, our problem setup follows a network quantization procedure \cite{jacob2018quantization,wang2022niti,yao2021hawq}: weights and inputs/outputs of activation functions are quantized to fixed-point numbers, but high-precision bias and intermediate operations (see \cref{sec:nn}). 

The universal representators in our results and floating-point results are both represented as a sum of indicator functions. However, the precise constructions of indicator functions are completely different due to the discrepancy between floating-point and fixed-point arithmetic we consider. Namely, those constructions are incompatible.
In the floating-point arithmetic in \cite{park2024expressive,hwang2025floatingpoint,hwang2025floating}, the rounding is performed after each addition/multiplication in affine transformations, i.e., a sequence of additions is non-associative. \cite{hwang2025floatingpoint,hwang2025floating} exploited this property to construct indicator functions and \cite{park2024expressive} highly relies on the properties of ReLU and the binary step function.
Furthermore, they also consider the overflow (i.e., if the output of some operation before rounding is too large/small, the rounded value will be $\pm\infty$); due to the threat of the overflow, they could not cover the domain of all finite floats, unlike our results.
On the other hand, we do not consider overflow as in the conventional quantization setup, and our construction of indicator functions may generate large absolute values before rounding, so that it can generate $\pm\infty$ if overflow exists.
Lastly, due to the different number formats in floating-point arithmetic and fixed-point arithmetic, if we directly transfer a network construction under one arithmetic to another, it may incur undesired round-off error and may not work as we want.}

\section{Proofs}\label{sec:proofs}
\subsection{Proof of \cref{thm:necessity}}\label{sec:pfthm:necessity}
As we described in \cref{sec:counterexample}, $\mcN_{\sigma,p,s,b}=\Qps$ for some $b\in\Qis$ is necessary for $\sigma$ quantized networks of more than one layer to universally represent. 
Specifically, we showed that if $\mcN_{\sigma,p,s,b}\ne\Qps$, then $\sigma$ quantized networks of more than one layer cannot represent a function whose range is $\Qps$.
Hence, in this proof, we show that one-layer $\sigma$ quantized network (i.e., rounded version of an affine transformation) cannot also represent some function with range $\Qps$ if $\mcN_{\sigma,p,s,b}\ne\Qps$.
Consider a bijective function $f:\Qps^d\rightarrow\Qps$ defined as follows: for $\bfx\in\Qps^d$ with $x_1=\frac{i}{s}$,
\begin{equation*}
    f\left(\bfx\right) = \begin{cases}
        \frac{i+1}{s}  &\text{ if  } i\equiv 1  \Mod{2}  \text{ and }i\neq 2^p-1,
        \\ \frac{i-1}{s} &\text{ if  } i\equiv 0  \Mod{2},
        \\ \frac{2^p-1}{s} &\text{ if  } i =  2^p-1.
    \end{cases}
\end{equation*}
Then, $f\left(\Qps^d\right) = \Qps$.
If one-layer $\sigma$ quantized networks under $\Qps$ can represent $f$, 
then there exists an affine transformation $\rho:\bbR^d\to\bbR$ with quantized weights and bias such that $\round{\rho}=f$. However, since $\round{\rho}$ is monotone on with respect to the first coordinate but $f$ is not, we cannot have $\round{\rho}=f$. This completes the proof.

\subsection{Proof of \cref{lem:necessity}}\label{sec:pflem:necessity}
Under the conditions in \cref{lem:necessity}, we have $r \le 2^p-1$ since $r\in\bbN$, $sr \in \Qps,$ and $r \le sr \le \max \Qps = \frac{2^p-1}{s} \le 2^p - 1. $
Note that if $\mcN_{\sigma,p,s,b} = \Qps$, then it holds that
\begin{align}
    \left\{ q \modd{r} : \frac{q}{s} \in \mcN_{\sigma,p,s,b}  \right\}  
    &= \left\{ q \modd{r} : \frac{q}{s} \in \Qps  \right\} \nonumber \\
    &= \{ 0,1,..., r-1\}. \label{eq:sec52_0}
\end{align}
Let $\round{\sigma}(x) = q_x/s$. By the assumption $ \round{\sigma}(x) = q_x/s \in \mathbb{Z}$ and $r | (q_x/s) $, we have $q_x =  s r q_x'$ for some $q_x' \in \mathbb{Z}$. Define $\widetilde{\sigma}:\Qps\rightarrow\Qps$  as 
\begin{equation*}
    \widetilde{\sigma}(x) \defeq  \frac{q_x'}{s}.
\end{equation*}
Then, we have 
\begin{equation*}
    \round{\sigma(x)}(x) = sr \widetilde{\sigma}(x).
\end{equation*}
Recall that $ \mcN_{\sigma,p,s,b}$ is defined as
\begin{equation*}
    \mcN_{\sigma,p,s,b} = \left\{\round{b+\sum_{i=1}^nw_ix_i}:n\in\bbN_0,w_i\in\Qps,\; x_i\in\round{\sigma}(\Qps)~\forall i\in[n]\right\}.
\end{equation*}
For $i\in [n]$ and $x_i\in\round{\sigma}(\Qps)$, choose $y_i\in \Qps$ such that $x_i = \round{\sigma}(y_i)$. 
Then, we have 
\begin{equation*}
    x_i = \round{\sigma}(y_i) = sr\widetilde{\sigma}(y_i)
\end{equation*}
and 
\begin{equation*}
    \round{b+\sum_{i=1}^nw_ix_i} =  \round{b+\sum_{i=1}^nw_i sr\widetilde{\sigma}(y_i)} = \round{\frac{sb + \sum_{i=1}^n r(sw_i)(s\widetilde{\sigma}(y_i))}{s}}.
\end{equation*}
Here, one can observe that $sb + \sum_{i=1}^n r(sw_i)(s\widetilde{\sigma}(y_i))\in \bbZ$ and
\begin{equation*}
    sb + \sum_{i=1}^n r(sw_i)(s\widetilde{\sigma}(y_i)) \equiv sb  \Mod{r}.
\end{equation*}
Therefore, we have 
\begin{equation*}
    s\round{b+\sum_{i=1}^nw_ix_i}  \equiv \begin{cases}
         2^{p}-1 \Mod{r}   &\text{ if } b+\sum_{i=1}^nw_ix_i\geq 2^{p}/s,
        \\   -2^{p}+1 \Mod{r}   &\text{ if } b+\sum_{i=1}^nw_ix_i\leq -2^{p}/s,
        \\  sb \Mod{r}   &\text{ otherwise. }
    \end{cases} \quad 
\end{equation*}
By \cref{eq:sec52_0}, it must hold that 
\begin{align}
  \{ 0,1, ... , r-1\}  &=  \left\{ q \modd{r} : \frac{q}{s} \in \mcN_{\sigma,p,s,b}  \right\}\notag\\   
  &\subset \left\{ 2^{p}-1\modd{r}, -2^{p}+1\modd{r}, sb\modd{r}\right\}. \label{eq:sec52_1}
\end{align}
If $r\ge4$, \cref{eq:sec52_1} cannot be true.  
If $r=3$ and $2\mid p$, since $2^2 \equiv 1 \Mod{3}$, $2^p=(2^{2})^{\frac{p}{2}} \equiv (1)^{\frac{p}{2}} \equiv 1 \Mod{3}$, we have $2^p-1,-2^p+1 \equiv 0 \Mod{r}$. 
Therefore, %
\cref{eq:sec52_1} cannot be true
since
\begin{equation*}
 \{0,1,2\}\not\subset \{ 0 ,sb\modd{r}\}.
\end{equation*}
 Therefore, $ \mcN_{\sigma,p,s,b}\neq \Qps$. This completes the proof.

\subsection{Proof of \cref{lem:monotone}}\label{sec:pflem:monotone}
Without loss of generality, we assume that $\sigma$ is monotonically non-decreasing; for a monotonically non-increasing $\sigma$, on can consider $\widetilde{\sigma}(x) \defeq -\sigma(x)$ instead.
Since $\sigma$ is monotonically non-decreasing, so is $\round{\sigma}$.
Define $z\in \Qps$ as \tcv{the smallest value in $\Qps$ satisfying $ \round{\sigma}(z) = \round{\sigma}\left(\frac{2^p-1}{s}\right)$.} 
Since $\round{\sigma}$ is non-constant, $z \neq \frac{-2^p+1}{s} = \min \Qps$. Furthermore, since $\round{\sigma}$ is non-decreasing, it holds that $\round{\sigma}(x) = \round{\sigma}(z)$ for $x \ge z$ and $\round{\sigma}(x) < \round{\sigma}(z)$ for $x < z$.
This completes the proof.

\subsection{Proof of \cref{lem:indicator}}\label{sec:pflem:indicator}

Without loss of generality, we assume that $\sigma$ satisfies \cref{cond:sufficency} with $\alpha = \beta = 1$ and some $z\in\Qps$ with $z>-q_{\max}$ (recall $q_{\max}=\frac{2^{p}-1}s$), i.e., $\round{\sigma}(x)=\max_{x'\in\Qps}\round{\sigma}(x')$ if and only if $z\le x\in\Qps$.
To construct the desired indicator function, we first define \tcv{two-layer neural network} $\phi:\Qps^d\rightarrow \bbR$ as
\begin{align}
    \phi(\bfx) \defeq \sum_{i=1}^d  \bigg(&- \round{\sigma}\left(\round{x_i-\alpha_i+z}\right) + \round{\sigma}\left(q_{\max}\right) \nonumber \\
    \qquad\qquad&-\round{\sigma}\left(\round{-x_i+\beta_i + z}\right) +   \round{\sigma}\left(q_{\max}\right)\bigg),
    \label{eq:phi_in_lem_indicator}
\end{align}
where $z$ is from \cref{cond:sufficency}.
\tcv{Note that $\phi$ can be implemented as the composition:
\begin{equation*}
    \phi = \rho_2\circ\round{\sigma}\circ \round{\rho_1},
\end{equation*}
where
\begin{align*}
    \rho_1:\Qps^d&\to\Qps^{4d}
    \\ \bfx&\mapsto (x_1-\alpha_1 +z, q_{\max}, -x_1+\beta_1 + z, q_{\max},\dots,  -x_d+\beta_d + z,q_{\max}),
\end{align*}
and $ \rho_2(x_1,\dots, x_{4d}) = \sum_{i=1}^{4d}x_i$.
}
Then, one can observe that $\phi(\bfx)=0$ if $\bfx\in\mcC$ and $\phi(\bfx)\ge1/s$ otherwise.
Define $g:\Qps^d\rightarrow \bbR$ as 
\begin{align}
        g(\bfx) &\defeq -\round{\sigma}\left(\round{\phi(\bfx) + \left( z - \frac{1}{s} \right)}\right) + \round{\sigma}\left( q_{\max}\right) \label{eq:g_in_lem_indicator}\\
    &=\begin{cases}
            \round{\sigma}\left( q_{\max}\right) - \round{\sigma}\left(z - \frac{1}{s}\right)  &\text{ if }  \bfx\in\mcC,
            \\0  &\text { otherwise}.
        \end{cases} \nonumber
\end{align}
\tcv{Note that for $\bfx\notin \mcC$, $\phi(\bfx)\ge 1/s$, and thus $\phi(\bfx) + z - 1/s\ge z$. This implies $ \round{\sigma}\left(\round{\phi(\bfx) + \left( z - {1}/{s} \right)}\right) = \round{\sigma}\left( q_{\max}\right)$ by \cref{cond:sufficency}.
Here, we note that $z-\frac1s\in\Qps$ since $z>-q_{\max}$.
The function $g$ can be implemented as $\sigma$ quantized neural network of the form
\begin{equation*}
    g = \rho_3 \circ \round{\sigma} \circ \round{\rho_2'}\circ \round{\sigma} \circ \round{\rho_1},
\end{equation*}
where the affine transformation $\rho_2':\Qps^{4d}\to \Qps^2$ is defined as 
\begin{equation}
    \rho_2' = \left(\rho_2 + z-\frac{1}{s}, q_{\max}\right),
\end{equation}
and the linear transformation $\rho_3:\Qps^2\to \Qps$ is defined as $\rho_3(x_1,x_2) = -x_1 + x_2$.
}

For each $q\in \Qps$, define $f_{q,1}$, $f_{q,2}$, and $F_q:\Qps^d\rightarrow \Qps$ as follows:
\begin{align}
    f_{q,1}(\bfx) &\defeq \round{q + \sum_{m_q} q_{\max}\times g(\bfx)}, \nonumber\\
    f_{q,2}(\bfx) &\defeq \round{q}, \nonumber \\
    F_q(\bfx)&  \defeq   \round\sigma \left( f_{q,1}(\bfx) \right ) -  \round\sigma \left( f_{q,2}(\bfx) \right )  \nonumber \\
    &=
    \round\sigma\left(\round{q + \sum_{m_q} \left(q_{\max}\times g(\bfx)\right) }\right) - \round\sigma\left( q\right),
\label{eq:fq_in_lem_indicator_fq}
\end{align}
where $m_q$ is some natural number satisfying $m_q q_{\max}\times \left( \round{\sigma}\left(q_{\max}\right) -\round{\sigma}\left(z - \frac{1}{s}\right)\right) > 2q_{\max}$ and $\sum_{m_q}$ denotes the $m_q$ times repeated addition of the same nodes.
\tcv{As $\rho_3$ is a linear transformation whose weights lie in $\{\pm 1\}$, the function $q_{\max}\rho_3$ is again an affine transformation, $\rho_3':(x_1,x_2)\mapsto q_{\max}x_1 - q_{\max}x_2$. Then $q_{\max}g = \rho'_3 \circ \round{\sigma} \circ \round{\rho_2'}\circ \round{\sigma} \circ \round{\rho_1}$ is a three-layer neural network.
Then, $q+ \sum_{m_q}q_{\max}g(\bfx)$ can also be implemented by a three-layer quantized neural network (See \cref{eq:mtimes_summation} for more details on $\sum_m$).
} 
$F_q$ can be calculated as
\begin{equation}
    F_q(\bfx)=\begin{cases}
            \round{\sigma}\left( q_{\max}\right) - \round{\sigma}\left(q\right)  &\text{ if }  \bfx\in\mcC,
            \\0  &  \text{ otherwise}\bf.
    \end{cases}\label{eq:Fq_in_lem_indicator}
    \end{equation}
By the definition of $\mcS_{\sigma,p,s}^\circ$, for any $\gamma\in \mcS_{\sigma,p,s,}^\circ$, there exist $n\in \bbN_0$, $w_j\in \Qps$, and $v_j\in\mcV_{\sigma,p,s}$ such that 
\begin{equation*}
    \gamma =\sum_{j=1}^nw_jv_j.
\end{equation*}
Suppose that $v_j = \round{\sigma}(v_{1,j}) - \round{\sigma}(v_{2,j})$
for some $v_{1,j}, v_{2,j}\in \Qps$.
Then, since 
\begin{align*}
  - F_{v_{1,j}}(\bfx) + F_{v_{2,j}}(\bfx)    &=  -\left( \round{\sigma}(q_{\max} ) - \round{\sigma}(v_{1,j}) \right) - \left(\round{\sigma}(q_{\max} ) + \round{\sigma}(v_{2,j}) \right) \nonumber \\
  &= \round{\sigma}(v_{1,j}) - \round{\sigma}(v_{2,j}) = v_j ,
\end{align*}
if $\bfx\in \mcC$ and $- F_{v_{1,j}}(\bfx) + F_{v_{2,j}}(\bfx)=0$ otherwise, 
it holds that
\begin{align*}
 \psi(\bfx)\defeq\sum_{j=1}^n w_j \left(F_{v_{2,j}}(\bfx) - F_{v_{1,j}}(\bfx) \right)  
       &=\begin{cases}
       \gamma     &\text{ if } \bfx\in\mcC,
          \\ 0  &\text{ otherwise, }  \end{cases}
\end{align*}
i.e., $\psi$ is the desired indicator function.

We now explicitly write the affine transformation $\rho$ and the \tcv{three-layer} network $f$ so that $\psi=\rho\circ\round{\sigma}\circ f$.
We define $\rho:\Qps^{4n}\rightarrow \bbR$ as 
\begin{equation}
  \rho(\bfx) \defeq \langle (-w_1, w_1,w_1,-w_1, -w_2,w_2,w_2,-w_2,\dots,  -w_n) , \; \bfx \rangle , 
  \label{eq:rho_in_lem_indicator}
\end{equation}
where $\langle\cdot,\cdot\rangle$ denotes the inner product
We also define the three-layer $\sigma$ quantized network $f(\,\cdot\,;\Qps) : \Qps^d\rightarrow\Qps^{4n}$ as
\begin{align}
      f(\bfx;\Qps) \defeq & \left(   f_{v_{1,1},1}(\bfx), \quad  f_{v_{1,1},2}(\bfx), \quad f_{v_{2,1},1}(\bfx), \quad f_{v_{2,1},2}(\bfx), \right. \nonumber \\
       & \;\; f_{v_{1,2},1}(\bfx), \quad  f_{v_{1,2},2}(\bfx), \quad f_{v_{2,2},1}(\bfx), \quad f_{v_{2,2},2}(\bfx), \nonumber \\
       & \;\;  \dots, \nonumber \\
  & \; \left. f_{v_{1,n},1}(\bfx), \quad  f_{v_{1,n},2}(\bfx), \quad f_{v_{2,n},1}(\bfx), \quad f_{v_{2,n},2}(\bfx)   \right).\label{eq:f_in_lem_indicator} 
  \end{align} 
Then, we have 
\begin{align*}
    &\rho\circ\round{\sigma}\circ f(\bfx;\Qps) \\
    & = w_1 \times  \left( - \round{\sigma} \left( f_{v_{1,1},1}(\bfx) \right) +  \round{\sigma} \left( f_{v_{1,1},2}(\bfx) \right) + \round{\sigma} \left( f_{v_{2,1},1}(\bfx) \right)-  \round{\sigma} \left( f_{v_{2,1},2}(\bfx) \right)  \right) \\
    &\quad+ \dots  \\
    &\quad+ w_n \times  \left( - \round{\sigma} \left( f_{v_{1,n},1}(\bfx) \right) +  \round{\sigma} \left( f_{v_{1,n},2}(\bfx) \right) + \round{\sigma} \left( f_{v_{2,n},1}(\bfx) \right) -  \round{\sigma} \left( f_{v_{2,n},2}(\bfx) \right)  \right)  \\
    &= \sum_{j=1}^n w_j \left(F_{v_{2,j}}(\bfx) - F_{v_{1,j}}(\bfx) \right)\\ &=\gamma\times\indc{\mcC}{\bfx}.
\end{align*}
This completes the proof.

\subsection{Proof of \cref{thm:sufficiency}}\label{sec:pfthm:sufficiency}
For any $f:\Qps^d\rightarrow \Qps$, $f$ can be represented as the sum of indicator functions as follows:
\begin{equation*}
    f(\bfx) = \sum_{\bfv \in \Qps^d} f(\bfv)\times  \indc{\{\bfv \}}{\bfx}.
\end{equation*}
Then, by the assumption that $\mcS_{\sigma,p,s,b}=\Qps$, we have $f(\bfv)\in \mcS_{\sigma,p,s,b}$ for any $\bfv\in \Qps^d$.
By the definition of $\mcS_{\sigma,p,s,b}$, there exists $\gamma_{\bfv}\in \mcS_{\sigma,p,s}^{\circ}$ such that
\begin{equation*}
    \round{b + \gamma_{\bfv}} = f(\bfv).
\end{equation*}
By \cref{lem:indicator}, there exist %
an affine transformation without bias
$\rho_\bfv$ and a $\sigma$ quantized network $\phi_\bfv$ such that 
\begin{equation*}
     \rho_\bfv\circ\round{\sigma}\circ \phi_\bfv(\bfx)=\gamma_{\bfv}\times\indc{\{\bfv\}}{\bfx}.
\end{equation*}
Define $g(\cdot;\Qps): \Qps^d\rightarrow\Qps$ as
\begin{equation*}
    g(\bfx;\Qps)\defeq  \round{b + \sum_{\bfv \in \Qps^d}   \rho_\bfv\circ\round{\sigma}\circ \phi_\bfv(\bfx) }
     =  \round{b + \sum_{\bfv \in \Qps^d}   \gamma_{\bfv}\times\indc{\{\bfv\}}{\bfx} }.
\end{equation*}
Then, for each $\bfv\in \Qps^d$,
\begin{equation*}
   g(\bfv;\Qps) = \round{b+\gamma_\bfv}=f(\bfv)
\end{equation*}
Since $g$ is a $\sigma$ quantized network, the proof is completed.

\subsection{Proof of \cref{lem:sufficiency}}\label{sec:pflem:sufficiency}
\textbf{Case (A1).} Define the set $\Sigma$ as 
        \begin{equation*}
        \Sigma \defeq    \left\{s\round{\sigma}\left(\frac{k}{s}\right) \in \bbZ : k\in \bbZ, q_1\leq k\leq q_2 \right\}.
        \end{equation*}
        Because $ \left|\sigma\left(\frac{q_1}{s}\right)\right|, \left|\sigma\left(\frac{q_2}{s}\right)\right| \le \frac{2^p-1}{s}$ and  $ \left|\sigma\left(\frac{q_2}{s}\right) - \sigma\left(\frac{q_1}{s}\right)\right|\geq \frac{1}{s}$, by \cref{lemma:qps_basic}, it follows that $ \left|\round{\sigma}\left(\frac{q_2}{s}\right) - \round{\sigma}\left(\frac{q_1}{s}\right)\right|\geq  \frac{1}{s}$, and $\Sigma$ has at least two elements. 
        For integers $z_1, z_2\in \bbZ$ defined as $z_1 \defeq s\round{\sigma}\left(\frac{q_1}{s}\right)$ and $z_2 \defeq s\round{\sigma}\left(\frac{q_2}{s}\right)$, 
without loss of generality, assume that $z_2 > z_1$.
        Since $|\sigma'(x)|<1 $, for any $k\in \bbZ$ such that $q_1\leq k < q_2$, the following inequality holds:
        \begin{equation}\label{eq:1s}
             \left|\sigma\left(\frac{k+1}{s}\right) - \sigma\left(\frac{k}{s}\right)\right| < \frac{1}{s}.
        \end{equation}
        Thus, by \cref{lemma:qps_basic2}, it holds that
        \begin{equation*}
        \left|\round{\sigma}\left(\frac{k+1}{s}\right) - \round{\sigma}\left(\frac{k}{s}\right)\right|\leq   \frac{1}{s}.
        \end{equation*}
        Therefore, $\round{\sigma}\left(\frac{k+1}{s}\right) - \round{\sigma}\left(\frac{k}{s}\right)$ should be one of $\frac{1}{s}, 0$, or $\frac{-1}{s}$,
        Then, the following relation holds:
        \begin{equation*}
            \Sigma \supset \left\{z_1, z_1+1,\dots, z_2\right\}.
        \end{equation*}
        Since 
        \begin{equation*}
            \mathcal V_{\sigma,p,s}\supset \left\{\frac{z}{s} - \frac{z'}{s}: z,z'\in \Sigma \right\},
        \end{equation*}
        we have $\frac{1}{s}\in \mathcal V_{\sigma,p,s}$. This implies that $\Qps =  \mcS_{\sigma,p,s,b}$ and completes the proof.

\noindent\textbf{Case (A2).} The proof is almost identical to the proof for the case (A1).
The only difference is that the condition $|\sigma'(x)|< 1$ is replaced by $|\sigma'(x)|\leq 1$ and $\sigma(x)\geq 0$. 
Consequently, the inequality in \cref{eq:1s} becomes 
\begin{equation*}
              \left|\sigma\left(\frac{k+1}{s}\right) - \sigma\left(\frac{k}{s}\right)\right| \leq  \frac{1}{s}.
\end{equation*}
Generally, we cannot guaranty that
\begin{equation*}
     \left|\round{\sigma}\left(\frac{k+1}{s}\right) - \round{\sigma}\left(\frac{k}{s}\right)\right|\leq   \frac{1}{s},
\end{equation*}
due to the away from zero tie-breaking rule.
However, as $\sigma(x)\geq 0$, we can assure that the inequality holds.
The remaining part of the proof is identical to that of the case (A1).
Thus, the proof is completed.

\noindent\textbf{Case (A3).} Define the set $\Sigma$ as 
        \begin{equation*}
        \Sigma \defeq    \left\{s\round{\sigma}\left(\frac{k}{s}\right) \in \bbZ: k\in \bbZ, q_1\leq k\leq q_2 \right\}.
        \end{equation*}
        Because  $ \left| \sigma\left(\frac{q_1}{s}\right)\right|, \left|\sigma\left(\frac{q_2}{s}\right)\right| \le \frac{2^p-1}{s} $  and $ \left|\sigma\left(\frac{q_2}{s}\right) - \sigma\left(\frac{q_1}{s}\right)\right|< \frac{2(q_2 - q_1)-1}{s}$, by \cref{lemma:qps_basic2}, it follows that $ \left|\round{\sigma}\left(\frac{q_2}{s}\right) - \round{\sigma}\left(\frac{q_1}{s}\right)\right| \leq    \frac{2(q_2 - q_1)-1}{s}$.
        As $1< \sigma'(x) \leq 2$, for $k\in \bbZ$ such that $q_1\leq k < q_2$, the following inequality holds:
            \begin{equation*}
              \sigma\left(\frac{k+1}{s}\right) - \sigma\left(\frac{k}{s}\right) \geq \frac{1}{s}.
        \end{equation*}
         Thus, by \cref{lemma:qps_basic},
        \begin{equation*}
             \round{\sigma}\left(\frac{k+1}{s}\right) - \round{\sigma}\left(\frac{k}{s}\right) \geq   \frac{1}{s}.
        \end{equation*}
        This implies that there are exactly $q_2 - q_1 + 1$ elements in $\Sigma$ between $s\round{\sigma}\left(\frac{q_1}{s}\right)$ and $s\round{\sigma}\left(\frac{q_2}{s}\right)$ whose difference is smaller than $ \frac{2(q_2 - q_1)}{s}$.
        Therefore, by the pigeonhole principle, there exists at least one $k\in \bbZ$ such that $\round{\sigma}\left(\frac{k+1}{s}\right) - \round{\sigma}\left(\frac{k}{s}\right)= \frac{1}{s} $.
        This implies that $\frac{1}{s}\in \mathcal V_{\sigma,p,s}$, and hence, $\Qps =  \mcS_{\sigma,p,s,b}$. 
  This completes the proof.

We now present some technical lemmas used in the proof of \cref{lem:sufficiency}.
\begin{lemma}
\label{lemma:qps_basic}
 If $x_1,x_2 \in \bbR$ satisfy 
 \begin{align*}
     | x_2 - x_1 |  \ge \frac{1}{s}, \quad  |x_1|, |x_2| \le \frac{2^p-1}{s},
 \end{align*}
then  we have $\left| \round{x_2}_{\Qps} - \round{x_1}_{\Qps} \right| \ge \frac{1}{s} $. 
\end{lemma}
\begin{proof}
    Without loss of generality, assume $x_1 < x_2$. Suppose $\round{x_1}_{\Qps}=\round{x_2}_{\Qps} = \frac{k}{s}$ for some $k \in \bbZ$. We have the following cases. Note that the rounding rule is \textit{ties away from zero}. \\
    \textbf{Case $x_1<x_2<0$.}
    We have $ \frac{k}{s} - \frac{1}{2s} < x_1, x_2 \le \frac{k}{s} + \frac{1}{2s}$ leading to $ x_2 - x_1 < \frac{1}{s}$. \\
    \textbf{Case $x_1<0<x_2$.} 
We have $ \frac{k}{s} - \frac{1}{2s} < x_1, x_2 < \frac{k}{s} + \frac{1}{2s}$ leading to $ x_2 - x_1 < \frac{1}{s}$. \\
    \textbf{Case $0<x_1<x_2$.} We have $ \frac{k}{s} - \frac{1}{2s} \le x_1, x_2 < \frac{k}{s} + \frac{1}{2s}$ leading to $ x_2 - x_1 < \frac{1}{s}$. \\
Namely, if $\round{x_1}_{\Qps}=\round{x_2}_{\Qps}$, we can conclude that $ x_2 - x_1 < \frac{1}{s}$, which contradicts the assumption $|x_2-x_1|\ge\frac1s$. 
Hence, it holds that $ \round{x_2}_{\Qps} - \round{x_1}_{\Qps} \ge \frac{1}{s} $. This completes the proof.
\end{proof}

\begin{lemma}\label{lemma:qps_basic2}
    If $x_1,x_2 \in \bbR$ satisfy 
    \begin{align*}
      | x_2 - x_1 |  <  \frac{k}{s}, \quad  |x_1|,|x_2| \le \frac{2^p-1}{s}%
    \end{align*}
for some $k\in\bbN$, 
then $\left| \round{x_2}_{\Qps} - \round{x_1}_{\Qps} \right| \le \frac{k}{s} $. 
\end{lemma}
\begin{proof}
    Without loss of generality, assume $x_1 \le x_2$. We prove the lemma using the mathematical induction on $k$. \\
\textbf{Base step ($k=1$).} We consider the following cases. \\
\textbf{Case $ -\frac{n}{s} -\frac{1}{2s} < x_1 \le -\frac{n}{s} +\frac{1}{2s} $ for some $n \in \bbN$.} Since $x_1 \le x_2 < x_1 + \frac{1}{s}$, we have $ -\frac{n}{s} -\frac{1}{2s} < x_2 < -\frac{n-1}{s} +\frac{1}{2s}.$ Hence, it holds that $\round{x_1}_{\Qps} = -\frac{n}{s}, \round{x_2}_{\Qps} = -\frac{n}{s}$ or $-\frac{n+1}{s}$, which leads to 
$ \round{x_2}_{\Qps} - \round{x}_{\Qps} \le \frac{1}{s}$. \\
\textbf{Case $  -\frac{1}{2s} < x_1 < \frac{1}{2s} $.}
We have  $  -\frac{1}{2s} < x_2 < \frac{3}{2s} $.  Hence, it holds that $\round{x_1}_{\Qps} = 0, \round{x_2}_{\Qps} = 0$ or $\frac{1}{s}$, which leads to 
$ \round{x_2}_{\Qps} - \round{x}_{\Qps} \le \frac{1}{s}$. \\
\textbf{Case $  \frac{n}{s} -\frac{1}{2s} \le x_1 < \frac{n}{s}+\frac{1}{2s} $ for some $n \in \bbN$.} Since $x_1 \le x_2 < x_1 + \frac{1}{s}$, we have $ \frac{n}{s} -\frac{1}{2s} \le x_2 < \frac{n+1}{s} +\frac{1}{2s} $. Hence, it holds that $\round{x_1}_{\Qps} = \frac{n}{s}, \round{x_2}_{\Qps} = \frac{n}{s}$ or $\frac{n+1}{s}$, which leads to 
$ \round{x_2}_{\Qps} - \round{x}_{\Qps} \le \frac{1}{s}$.
Now we prove when $k=1$. 

\noindent\textbf{Inductive step.}
Suppose that the lemma holds for $k=1,2,\dots,m$ and  $ \frac{m}{s} \le x_2 - x_1  < \frac{m+1}{s}$. We consider the following cases. \\
\textbf{Case $x_1 \ne -\frac{1}{2s}$.} 
Let  $x' = x_1 + \frac{1}{s}$. Then we have \begin{align*}
        x_2 - x' < \frac{m}{s}, \quad x' - x_1 = \frac{1}{s}.
\end{align*}
By the induction hypothesis, we have 
\begin{align*}
    \round{x_2}_{\Qps} - \round{x'}_{\Qps} \le \frac{m}{s}.
\end{align*}
Since $x_1 \ne -\frac{1}{2s}$ and rounding is ``tie away from zero'', $x_1$ and $x'$ are rounded to the same direction leading to $\round{x'}_{\Qps} - \round{x_1}_{\Qps}=\frac{1}{s}$. Therefore,  
\begin{align*}
   \round{x_2}_{\Qps} - \round{x_1}_{\Qps}  &= (\round{x_2}_{\Qps} - \round{x'}_{\Qps}) + (\round{x'}_{\Qps} - \round{x_1}_{\Qps}) \\
   &\le \frac{m}{s}+\frac{1}{s} = \frac{m+1}{s}.
\end{align*}
\textbf{Case $x_1 = -\frac{1}{2s}$.} 
Let $x' = x_2 - \frac{1}{s}$. Since $x_2 \ge \frac{3}{2s}$, and $x' \ge \frac{1}{2s}$, $x_2$ and $x'$ are rounded to the same direction leading to $\round{x_2}_{\Qps} - \round{x'}_{\Qps}=\frac{1}{s}$. 
By the induction hypothesis, we have 
\begin{align*}
    \round{x'}_{\Qps} - \round{x_1}_{\Qps} \le \frac{m}{s}.
\end{align*}
Therefore,  
\begin{align*}
   \round{x_2}_{\Qps} - \round{x_1}_{\Qps}  &= (\round{x_2}_{\Qps} - \round{x'}_{\Qps}) + (\round{x'}_{\Qps} - \round{x_1}_{\Qps}) \\
   &\le \frac{1}{s}+\frac{m}{s} = \frac{m+1}{s},
\end{align*}
which completes the proof.
\end{proof}

\subsection{Proof of \cref{lem:activation}}\label{sec:pflem:activation}
For $\relu$ and $\ELU$, we have $\sigma(x) = x$ for $x \ge 0$. Hence, we have  $\round{\sigma}(0)=0$, $\round{\sigma}(\frac{1}{s})=\frac{1}{s}$. This implies that $\frac{1}{s}\in \mathcal V_{\sigma,p,s}$, and $\Qps =  \mcS_{\sigma,p,s,b}$. 

For the remaining activation functions, we use \cref{lem:sufficiency2}, which is stated in the end of this section with its proof.
We now consider $\SiLU$, $\Mish$, and $\GeLU$ under the assumption that $p\ge3$.
For $s \ge 3$, one can observe that $\SiLU$, $\Mish$, and $\GeLU$ satisfy the assumption in the case (A4) in \cref{lem:sufficiency2}: $\sigma'\left((0,\frac{2}{3})\right) \subset (\frac{1}{2},1)$ (see \cref{table:lip}) and $\sigma(\frac{2}{s}) \le \sigma(0)+\frac{2}{s} \le \frac{2^p-1}{s} \in \Qps$ since $\sigma'(x) \le 1$ for $ 0 < x < \frac{2}{3} $. For $s=1$, we can verify that $\SiLU$, $\Mish$, and $\GeLU$ satisfy the assumption in the case (A6) in \cref{lem:sufficiency2} by checking $\sigma\left((3,5)\right) \subset (0,7)$ and $\sigma'\left((3,5)\right) \subset (1,1.5)$  (see \cref{table:lip}). For $s=2$, we can verify that  $\SiLU$, $\Mish$, and $\GeLU$ satisfy the assumption in the case (A6) in \cref{lem:sufficiency2} by checking $\sigma((1.5,2.5)) \subset (0,3.5)$ and $\sigma'((1.5,2.5)) \subset (1,1.5)$  (see \cref{table:lip}).

We next verify $\Sigmoid$. For $ 3 \le s \le 2^p-1 \le \frac{2^p-4}{\Sigmoid(0)} = 2^{p+1}-8 $, we verify that $\Sigmoid$ satisfy the assumption in the case (A5) in \cref{lem:sufficiency2}: $\sigma'\left((-1,1)\right) \subset (\frac{1}{6}, 1)$ (see \cref{table:lip_softplus}) and $\sigma(\frac{3}{s}) \le \sigma(0)+\frac{3}{s} \le \frac{2^p-1}{s}$ since $\sigma'(x) \le 1$ for $ -1 < x < 1$. For $s=1$, we have  $\round{\sigma}(0)=\round{\frac{1}{2}}=1$, $\round{\sigma}(-1)=\round{\frac{1}{1+e}}=0$. Hence, $1=\frac{1}{s}\in \mathcal V_{\sigma,p,s}$, and then, $\Qps =  \mcS_{\sigma,p,s,b}$. For $s=2$, we have  $\round{\sigma}(0)=\round{\frac{1}{2}}=\frac{1}{2}$, $\round{\sigma}(-1)=\round{\frac{1}{1+e}}=0$. Hence, $\frac{1}{2}=\frac{1}{s}\in \mathcal V_{\sigma,p,s}$ and $\Qps =  \mcS_{\sigma,p,s,b}$. 

We lastly consider $\SoftPlus$.
For $1 \le s \le  \frac{2^p-4}{\SoftPlus(0)} \approx \frac{10}{7}(2^p-4)$, we verify that $\SoftPlus$ satisfy (A7):  $\sigma'\left((0,3)\right) \subset (\frac{1}{2}, 1)$ (see \cref{table:lip_softplus}) and $\sigma(\frac{3}{s}) \le \sigma(0)+\frac{3}{s} \le \frac{2^p-1}{s}$ since $\sigma'(x) \le 1$ for $ 0 < x < 3$. Note that we have $2^p-1 \le \frac{2^p-4}{\SoftPlus(0)}$ when $p\ge 4$ and  $\frac{2^p-4}{\SoftPlus(0)}  \approx 5.7$ when $p=3$.

\begin{lemma}\label{lem:sufficiency2}
Let $\sigma:\bbR\rightarrow\bbR$ be a continuous function and $p,s\in\bbN$ satisfying $p\ge3$ and $s\le2^p-1$. 
If $\Qps$ and $\sigma$ satisfy one of the following conditions, then $\mcS_{\sigma,p,s,b}=\Qps$.
    \begin{enumerate}[leftmargin=0.4in]
        \item[(A4)] $\sigma$ is differentiable on $\left(0,\frac{2}{s}\right)$, $ \frac{1}{2} \le \sigma'(x) < 1 $ and $|\sigma(x)|\leq \frac{2^p-1}{s}$ for $x \in \left(0,\frac{2}{s}\right)$.
        \item[(A5)] $\sigma$ is differentiable on $\left(-\frac{3}{s},\frac{3}{s}\right)$, $ \frac{1}{6} \le \sigma'(x) < 1 $ and $|\sigma(x)|\leq \frac{2^p-1}{s}$ for $x \in \left(-\frac{3}{s},\frac{3}{s}\right)$. 
        \item[(A6)] $\sigma$ is differentiable on $\left(\frac{3}{s},\frac{5}{s}\right)$, $ 1 \le \sigma'(x) < \frac{3}{2} $ and $| \sigma(x)| \leq \frac{2^p-1}{s}$ for $x \in \left(\frac{3}{s},\frac{5}{s}\right)$. 
        \item[(A7)] $\sigma$ is differentiable on $\left(\frac{1}{s},\frac{3}{s}\right)$, $ \frac{1}{2} < \sigma'(x) <  1 $ and $| \sigma(x) | \leq \frac{2^p-1}{s}$ for $x \in \left(\frac{1}{s},\frac{3}{s}\right)$. 
    \end{enumerate}
\end{lemma}
\begin{proof}
\noindent\textbf{Case (A4).}
Since $ \frac{1}{2} \le \sigma'(x) < 1 $ for $0 < x < \frac{2}{s}$, by the mean value theorem,
        \begin{equation*}
            \sigma\left(\frac{2}{s}\right)- \sigma\left(0\right)\ge \frac{1}{s},
        \end{equation*}
    which satisfies the assumption in the case (A1) in \cref{lem:sufficiency} with $q_1=0$ and $q_2=\frac2s$.

\noindent\textbf{Case (A5).} Since $ \frac{1}{6} \le \sigma'(x) < 1 $ for $ -\frac{3}{s} < x < \frac{3}{s}$, by the mean value theorem,
\begin{equation*}
        \sigma\left(\frac{3}{s}\right)- \sigma\left( -\frac{3}{s}\right) \ge \frac{1}{s},
    \end{equation*}
    which satisfies the assumption in the case (A1) in \cref{lem:sufficiency} with $q_1 =-3$ and $q_2 = 3$.

\noindent\textbf{Case (A6).} Since $ 1 \le \sigma'(x) < \frac{3}{2} $ for $ \frac{3}{s} < x < \frac{5}{s}$, by the mean value theorem, 
    \begin{equation*}
        \sigma\left(\frac{5}{s}\right)- \sigma\left( \frac{3}{s}\right) <  \frac{2}{s} \times \frac{3}{2} = \frac{3}{s} < \frac{2(5-2)-1}{s} ,
    \end{equation*}
    which satisfies the assumption in the case (A3) in \cref{lem:sufficiency} with $q_1 =3$ and $q_2 = 5$.

\noindent\textbf{Case (A7).} Since $ \frac{1}{2} \le \sigma'(x) < 1 $ for $ \frac{1}{s} < x < \frac{3}{s}$, by the mean value theorem,
        \begin{equation*}
            \sigma\left(\frac{3}{s}\right)- \sigma\left( \frac{1}{s}\right) \ge \frac{1}{s},
        \end{equation*}
        which satisfies the assumption in the case (A1) in \cref{lem:sufficiency} with $q_1 =1$ and $q_2 = 3$.
\end{proof}

\begin{table}[h]
\footnotesize
\caption{%
Properties of various activation functions for verifying the conditions in \cref{lem:sufficiency}. Each entry in the table is an open interval (rounded to two decimal places) that contains the value.
}
\vspace{0.3cm}
\centering
\begin{tabular}{@{}ccccccccc@{}}
\toprule
\begin{tabular}{@{}c@{}}Activation \\ function\end{tabular} &     \begin{tabular}{@{}c@{}} $\sigma((0, \frac{2}{3}))$ \\   \end{tabular}   & \begin{tabular}{@{}c@{}} $\sigma'((0, \frac{2}{3}))$ \\   \end{tabular} & \begin{tabular}{@{}c@{}} $\sigma((3, 5))$ \\   \end{tabular}  & \begin{tabular}{@{}c@{}} $\sigma'((3, 5))$ \\   \end{tabular} & \begin{tabular}{@{}c@{}} $\sigma((\frac{3}{2}, \frac{5}{2}))$ \\   \end{tabular} & \begin{tabular}{@{}c@{}} $\sigma'((\frac{3}{2}, \frac{5}{2}))$ \\   \end{tabular}  \\ 
\midrule
SiLU &  (0, 0.45) & (0.5, 0.82) &(2.85, 4.97)  & (1.02, 1.09) & (1.22, 2.32) & (1.04, 1.10)  \\ 
Mish &  (0, 0.53) & (0.6, 0.96) &(2.98, 5.00)  & (1.00, 1.03) & (1.40, 2.48) & (1.04, 1.09) \\ 
GELU &   (0, 0.50) & (0.5, 0.97) &(2.99, 5.00)  & (1.00, 1.02) & (1.22, 2.32) & (1.04, 1.10)\\ 
\bottomrule
\end{tabular}
\label{table:lip}
\end{table}

\begin{table}[h]

\captionsetup{width=0.9\linewidth}

\caption{Properties of $\SoftPlus$ and $\Sigmoid$ used for verifying the conditions of \cref{lem:sufficiency}. Each entry in the table is an open interval (rounded to two decimal places) that contains the value.}

\vspace{0.3cm}
\footnotesize
\begin{tabular}{@{}ccccccc@{}}
\toprule
\begin{tabular}{@{}c@{}}Activation \\ function\end{tabular} &  \begin{tabular}{@{}c@{}}$\sigma'\left((0,3)\right)$ \\  \end{tabular} & \begin{tabular}{@{}c@{}}Activation \\ function\end{tabular} & \begin{tabular}{@{}c@{}}$\sigma'\left((-1,1)\right)$ \\  \end{tabular} \\ 
\midrule
$\SoftPlus$ & (0.50, 0.96) & $\Sigmoid$ & (0.2, 0.25) \\ 
\bottomrule
\end{tabular}

\label{table:lip_softplus}
\vspace{0.2in}

\end{table}

\subsection{Proof of \cref{lem:nece_suff}}\label{sec:pflem:nece_suff}
Recall that
\begin{align*}
\mcN_{\sigma,p,s,b}&\defeq\left\{\round{b+\sum_{i=1}^nw_ix_i}:n\in\bbN_0,w_i\in\Qps,x_i\in\round{\sigma}(\Qps)~\forall i\in[n]\right\},\\
\mathcal V_{\sigma,p,s}&\defeq\left\{\round{\sigma}(x)-\round{\sigma}(y):x,y\in\Qps\right\},\\
\mcS_{\sigma,p,s,b}&\defeq\left\{\round{b+\sum_{i=1}^nw_ix_i}:n\in\bbN_0,w_i\in\Qps,x_i\in\mcV_{\sigma,p,s}~\forall i\in[n]\right\}.
\end{align*}
If there exists $x \in \Qps$ such that $\round{\sigma}(x) \in \mcV_{\sigma, p,s}$, then there exist 
$y,z\in\Qps$ such that
\begin{equation*}
    \round{\sigma}(x) = \round{\sigma}(y) - \round{\sigma}(z).
\end{equation*}
Then, for any $w\in\Qps$, 
\begin{equation*}
    \round{\sigma}(w) = \left(\round{\sigma}(w) - \round{\sigma}(x)\right) + \left(\round{\sigma}(y) - \round{\sigma}(z)\right).
\end{equation*}
Therefore, %
$\mcS_{\sigma,p,s,b}\supset \mcN_{\sigma,p,s,b} $. 
Since $\mcS_{\sigma,p,s,b}\subset \mcN_{\sigma,p,s,b} $ by their definitions, we have $\mcS_{\sigma,p,s,b}= \mcN_{\sigma,p,s,b} $.
We can prove $\mcS_{\sigma,p,s,b}= \mcN_{\sigma,p,s,b}$ when $\round{\sigma}(x)=0$ for some $x\in\Qps$ using the same argument.
This completes the proof.

\subsection{Proof of \cref{thm:b_necessity}}\label{sec:pfthm:b_necessity}
\begin{proof}
    Since the proof of \cref{thm:b_necessity} is almost identical to that of \cref{thm:necessity}, we omit the proof.
\end{proof}

\subsection{Proof of \cref{lem:b_necessity}}\label{sec:pflem:b_necessity}
\begin{proof}
This proof is similar to the proof of \cref{lem:necessity}.
By the assumption of the lemma, there exists a function $\widetilde{\sigma}:\Qps\rightarrow\Qps$ such that 
\begin{equation*}
    \round{\sigma(x)} = r \widetilde{\sigma}(x).
\end{equation*}
Recall that $ \mcB\mcN_{\sigma,p,s,b}$ is defined as
\begin{equation*}
    \mcB\mcN_{\sigma,p,s,b} = \left\{\round{b+\sum_{i=1}^nw_ix_i}:n\in\bbN_0,w_i\in\{-1,1\},x_i\in\round{\sigma}(\Qps)~\forall i\in[n]\right\}.
\end{equation*}
For $i\in [n]$ and $x_i\in\round{\sigma}(\Qps)$, consider $y_i\in \Qps$ such that $x_i = \round{\sigma}(y_i)$. 
Then, 
\begin{equation*}
    x_i = \round{\sigma}(y_i) = r\widetilde{\sigma}(y_i),
\end{equation*}
and 
\begin{equation*}
    \round{b+\sum_{i=1}^nw_ix_i} =  \round{b+\sum_{i=1}^nw_i r\widetilde{\sigma}(y_i)} = \round{\frac{sb + \sum_{i=1}^n rw_i(s\widetilde{\sigma}(y_i))}{s}}.
\end{equation*}
Here, we have $sb + \sum_{i=1}^n rw_i(s\widetilde{\sigma}(y_i))\in \bbZ$, and
\begin{equation*}
    sb + \sum_{i=1}^n rw_i(s\widetilde{\sigma}(y_i)) \equiv sb \quad (\text{mod } r).
\end{equation*}
Therefore, %
\begin{equation*}
    s \round{b+\sum_{i=1}^nw_ix_i} \equiv \begin{cases}
        2^{p}-1 {\Mod{r}} &\text{ if } b+\sum_{i=1}^nw_ix_i\geq 2^{p}/s ,
        \\ -2^{p}+1 {\Mod{r}} &\text{ if } b+\sum_{i=1}^nw_ix_i\leq -2^{p}/s ,
        \\ sb {\Mod{r}} &\text{ otherwise }
    \end{cases} 
\end{equation*}
Thus, as in the proof of \cref{lem:necessity} we can conclude that $\mcB\mcN_{\sigma,p,s,b}\neq \Qps$. This completes the proof.
\end{proof}
\subsection{Proof of \cref{lem:b_indicator}}
We follow the proof outline of \cref{lem:indicator}.
First, we reuse $g$ defined in \cref{eq:g_in_lem_indicator} since it only uses binary weights.
However, since $F_q$ defined in \cref{eq:Fq_in_lem_indicator} uses non-binary weights, we redefine $F_q$ as follows:
\begin{equation*}
    F_q(\bfx)\defeq \round\sigma\left(\round{q + \sum_{m_q} g(\bfx) }\right) - \round\sigma\left( q\right),
\end{equation*}
where $m_q\in \bbN$ is a natural number satisfying $m_q>2^{p+1}-2$, and $\sum_{m_q}$ denotes the $m_q$ times repeated additions of the same values ($g(x)$ in the above equation). Then, $g$ and $F_q$ become networks with binary weights and have the same outputs as in the proof of \cref{lem:indicator}.

Then, by the definition of $\mcB\mcS_{\sigma,p,s}^\circ$, for any $\gamma\in \mcB\mcS_{\sigma,p,s}^\circ$, there exist $n\in \bbN_0$, $w_i\in \{-1,1\}$, and $v_i\in\mcV_{\sigma,p,s}$ for $i\in [n]$ such that 
\begin{equation*}
    \gamma =  \sum_{i=1}^nw_iv_i.
\end{equation*}
If we define $\rho$ and $f$ as in \cref{eq:rho_in_lem_indicator} and \cref{eq:f_in_lem_indicator}, respectively, then $\rho\circ\round{\sigma}\circ f(\bfx;\Qps)=\gamma\times\indc{\mcC}{\bfx}$, and the proof is completed.

\subsection{Proof of \cref{thm:b_sufficiency}}
    The proof is identical to the proof of \cref{thm:sufficiency} except that $\mcS_{\sigma,p,s,b}$ in the proof is replaced by $\mcB\mcS_{\sigma,p,s,b}$ and \cref{lem:indicator} is replaced by \cref{lem:b_indicator}.

\subsection{Proof of \cref{lem:b_sufficiency}}\label{sec:pflem:b_sufficiency}
    As the proof construction of \cref{lem:sufficiency} only uses binary coefficients, the same proof applies to \cref{lem:b_sufficiency}.

\subsection{Proof of \cref{lemma:bchange}}\label{sec:pflem:bchange}
In this proof, we first show that $\mcS_{\sigma,p,s,b}=\mcS_{\sigma,p,s,b'}$ for some $|b'|\le q_{\max}$. We then show that 
$$\mcS^\ast_{\sigma,p,s,b',2q_{\max}+\frac{3}{2s}}\defeq\left\{ \round{\gamma + b'} : \gamma \in \mcS_{\sigma,p,s}^\circ, \;  |\gamma| \le 2q_{\max}+\frac3{2s}\right\}=\Qps.$$

Note that for $\tilde{v} \in \mcV_{\sigma,p,s}$, we have $|\tilde{v}| \le 2 q_{\max}$. 
For $\tilde{b} = b-\tilde{v}, n \in \bbN_{0}, w_i \in \Qps, v_i \in \mcV_{\sigma,p,s}$, since
\[ \round{ \sum_{i=1}^n w_i v_i+b }  =  \round{ \sum_{i=1}^n w_i v_i + \tilde{v}+ (b-\tilde{v}) } = \round{ \sum_{i=1}^n w_i v_i + \tilde{v}+ \tilde{b} }, \]
we have  $\mcS_{\sigma,p,s,b} = \mcS_{\sigma,p,s,\tilde{b}}$. Hence, $\mcS_{\sigma,p,s,b} = \mcS_{\sigma,p,s,b_m}$ for  $b_m = b + m \tilde{v}$ where $m \in \bbZ$. Since $ |\tilde{v}| \le 2 q_{\max}$, there exists some $m'$ such that $|b + m' \tilde{v}| \le q_{\max}$. Let $b' \defeq b + m' \tilde{v}$; then, we have $\mcS_{\sigma,p,s,b} = \mcS_{\sigma,p,s,b'}$. 

Now we show $\mcS^\ast_{\sigma,p,s,b',2q_{\max}+\frac{3}{2s}} = \mcS_{\sigma,p,s,b'} = \Qps$. 
First, consider the case that 
\[ \round{ \sum_{i=1}^n w_i v_i +  b' } \ne \pm q_{\max}. \]
Then, it holds that 
\[ -q_{\max} + \frac{1}{2s} <  \sum_{i=1}^n w_i v_i +  b' <  q_{\max} - \frac{1}{2s}, \]
which leads us to 
\begin{align}
    \left|\sum_{i=1}^n w_i v_i\right| \le 2 q_{\max}. 
    \label{eq: gammale2qmax}
\end{align}
We now prove the remaining case: there exists $\gamma \in \mcS_{\sigma,p,s}^\circ$ such that $ \round{\gamma + b'} = \pm q_{\max}$ and $|\gamma| \le 2 q_{\max} + \frac{3}{2s}$. \\ 
Pick $\gamma_1, \gamma_2 \in \mcS_{\sigma,p,s}^\circ$ such that $\round{\gamma_1 + b'} = q_{\max} - \frac{1}{s}$ and  $\round{\gamma_2 + b'} = q_{\max} - \frac{2}{s}$. Then we have
\begin{align*}
    q_{\max} -\frac{3}{2s} &\le \gamma_1 + b' < q_{\max} -\frac{1}{2s}, \\    q_{\max} -\frac{5}{2s} &\le \gamma_2 + b' < q_{\max} -\frac{3}{2s}
\end{align*}
Hence
\begin{align*}
 0  <  \gamma_1 - \gamma_2 < \frac{2}{s}.
\end{align*}
Let 
$\hat{\gamma} \defeq \gamma_1 - \gamma_2 \in \mcS_{\sigma,p,s}^\circ$. If $0< \hat{\gamma} < \frac{1}{s}$, there exists $n_{\hat{\gamma}}$ such that $ \frac{1}{s} \le n_{\hat{\gamma}}\hat{\gamma} < \frac{2}{s}$. 
Let $$ g \defeq \begin{cases}
    n_{\hat{\gamma}}\hat{\gamma} \; &\text{if} \; 0< \hat{\gamma} < \frac{1}{s},\\
    \hat{\gamma} \; &\text{if} \; \frac{1}{s}\le \hat{\gamma} < \frac{2}{s}.
\end{cases}$$
Then, we have $\frac{1}{s}\le g  < \frac{2}{s}$ and $\round{(\gamma_1+g)+b'}=q_{\max}$ since
\begin{align*}
    q_{\max} -\frac{1}{2s}\le(\gamma_1+g)+b'.
\end{align*}
Furthermore, it holds that
\begin{align*}
  -\frac{1}{2s}  \le q_{\max} -\frac{3}{2s} + g -b' \le \gamma_1+g < q_{\max} -\frac{1}{2s} + g - b' \le 2q_{\max}+\frac{3}{2s}. 
\end{align*}
Therefore, we have $\gamma_{+}\in \mcS_{\sigma,p,s}^\circ$ such that $ \round{\gamma_{+} + b'} =  q_{\max}$ and $|\gamma_{+}| \le 2 q_{\max} + \frac{3}{2s}$. 
Similarly, we have $\gamma_{-} \in \mcS_{\sigma,p,s}^\circ$ such that $ \round{\gamma_{-} + b'} =  q_{\max}$ and $|\gamma_{-}| \le 2 q_{\max} + \frac{3}{2s}$.
This completes the proof.

\subsection{Proof of \cref{lem:param_count}}\label{sec:pflem:param_count}
In this proof, we count the maximum number of parameters in our indicator function construction in \cref{lem:indicator} by using \cref{lem:gamma_count,lem:bezout_multiple} that are presented in the end of this section. 

First, observe that {$12d$} parameters are used to construct $\phi(x)$ in \cref{eq:phi_in_lem_indicator}, and therefore, 
{$(12d+5)$} parameters are used to construct $g(\bfx)$ in \cref{eq:g_in_lem_indicator}. 
Since $m_q = 2s+1$ suffices to guarantee $m_q q_{\max}\times \left( \round{\sigma}\left(q_{\max}\right) -\round{\sigma}\left(z - \frac{1}{s}\right)\right) > 2q_{\max}$ in the proof of \cref{lem:indicator}, 
$m_q(12d+5)+4\le (2s+1)(12d+5)+4$ parameters are sufficient for constructing $F_q(\bfx)$ in \cref{eq:fq_in_lem_indicator_fq}.
Since $s\round{\sigma}(\bfx)\in\{-2^{p-1}+1,-2^{p-1}+2,\dots,2^{p-1}-1\}$, we have $|\mcV_{\sigma,p,s}|<2^{p+2}$.
By \cref{lem:gamma_count}, 
for each $\gamma\in\mcS_{\sigma,p,s}^\circ$, there exists $n\le s2^{2p+4}$, $w_1,\dots,w_n\in\Qps$, and $v_1,\dots,v_n\in\mcV_{\sigma,p,s}$ such that $\gamma=\sum_{j=1}^nw_jv_j$.
Since our construction of the indicator function uses at most $O(n\delta)$ parameters where $\delta=(2s+1)(12d+5)+4$ denotes the maximum number of parameters used in $F_q$, we can conclude that $O(2^{2p}s^2d)$  parameters are sufficient for our indicator function construction in \cref{lem:indicator}.

\begin{lemma}\label{lem:gamma_count}
Let $\gamma\in \mcS_{\sigma,p,s}^\circ$ with ${|\gamma|} \le 2 q_{\max} + {\frac{3}{2s}}$. Then there exists $n\in \bbN_0$ such that $n \le{  s (2^{p+2}-1)\left|\mcV_{\sigma,p,s}\right|}  $, ${u_i}\in \Qps$, and ${z_i}\in\mcV_{\sigma,p,s}$ for $i\in [n]$ such that 
    \begin{equation*}
    \gamma =  \sum_{i=1}^n{u_i}{z_i}.
\end{equation*}
\end{lemma}
\begin{proof}
    Without loss of generality, assume $\gamma \ge 0$. 
    We rewrite $\gamma$ as  $\gamma = \frac{\gamma_0}{s^2}$ for some $\gamma_0 \in \bbZ$. 
    Let $m = |\mcV_{\sigma,p,s}{=\{v_1,\dots,v_m\}}|$, $v_i = \frac{x_i}{s}$ for $v_i \in \mcV_{\sigma,p,s}$ and $d = \gcd (x_1, \dots , x_m)$. Then, we have $|x_1|, \dots , |x_m| \le { 2s q_{\max} = 2^{p+1}-2}$. By \cref{lem:bezout_multiple}, there exist $c_1, \dots , c_m \in \bbZ$ such that $\sum_{i=1}^{m} c_i x_i = d$ where $| c_i| \le \max_{i=1,\dots,m}| x_i| \le 2^{p+1}-2 $.  Since $|\frac{c_i}{s}|  \le {\frac{2^{p+1}-2}{s}}$, there exists $u_{i,1},u_{i,2} \in \Qps$ such that {$u_{i,1}+u_{i,2}=\frac{c_i}s$}. %
   Then, we have 
    \begin{align*}
       \sum_{i=1}^{m} \sum_{j=1}^{2} \frac{u_{i,j}}{s} \times \frac{{x_{i}}}{s}  =  \sum_{i=1}^{m} \frac{c_i}{s} \times \frac{x_i}{s}= \frac{d}{s^2}.
    \end{align*}
 Next, let ${w_{(j-1)m+i}} \defeq u_{i,j}/s$ for $i=1,\dots,m$, $j=1,2$. Then, it holds that
\begin{align*}
   {\left(\sum_{i=1}^{m}  w_{i} \times \frac{x_{i}}{s}\right)+\left(\sum_{i=1}^{m}  w_{i+m} \times \frac{x_{i}}{s}\right) =   \left(\sum_{i=1}^{m}  w_{i}  v_i\right)+\left(\sum_{i=1}^{m}  w_{i+m}  v_i\right)} = \frac{d}{s^2}.
\end{align*}
Since $d | \gamma_0$, we have
\begin{align*}
   \sum_{\gamma_0/d} \left({\left(\sum_{i=1}^{m}  w_{i}  v_i\right)+\left(\sum_{i=1}^{m}  w_{i+m}  v_i\right)} \right) = \frac{\gamma_0}{d} \times   \frac{d}{s^2} = \frac{\gamma_0}{s^2}={\gamma},
\end{align*}
where $\sum_{\gamma_0/d}$ denotes the addition of $\frac{\gamma_0}{d}$ identical terms. 
{Since $ |\gamma_0| = {s^2 |\gamma| \le  s(2q_{\max}+\frac{3}{2})=s(2^{p+1}-\frac{1}{2})}$ and $d\ge 1$, by the above equation, one can observe that there exist $n=2m\gamma_0/d\le {s(2^{p+2}-1)|\mcV_{\sigma,p,s}|}$, $u_1,\dots,u_n\in\Qps$, and $z_1,\dots,z_n\in\mcV_{\sigma,p,s}$ such that $\gamma=\sum_{i=1}^nu_iz_i$. This completes the proof.} 
\end{proof}

\begin{lemma}\label{lem:bezout_multiple}
    Let $x_1< \dots < x_n \in \bbN$ be natural numbers with $\gcd(x_1,\dots,x_n)=d$. Then there exists $c_1, \dots, c_n \in \bbZ$ such that 
    \begin{align*}
        \sum_{i=1}^n c_i x_i = d,
    \end{align*}
where 
\begin{align*}
      |c_1| \le \frac{x_n}{d}, \quad |c_i| \le \frac{x_1}{d}, \quad \forall i = 2, \dots ,n.
\end{align*}
\end{lemma}
\begin{proof}
Before starting the proof, we note that \cref{lem:bezout_multiple} is an extension of the Bézout's identity \cite{bezout1779theorie}.
Since if $\min(x_1,\dots,x_n)=1$, then the statement naturally follows,
    without loss of generality, we assume $1 < x_1<\dots<x_n$ and $d=1$. Since $d=\gcd(x_1,\dots,x_n)=1$, there exists $b_1,\dots,b_n\in\bbZ$ such that $\sum_{i=1}^n b_ix_i=1$. Let $k_n \in \bbZ$ such that $ |b_n + k_n x_1| \le |x_1| $ and let $c_n = b_n+k_nx_1$. Then we have \begin{align*}
        (b_1 - k_n x_n) x_1 + b_2 x_2 + \dots + b_{n-1} x_{n-1} + (b_n+ k_n x_1) x_n = 1.  
    \end{align*}
    Next, pick $k_{n-1} \in \bbZ$ such that $|b_{n-1} +k_{n-1} x_1 | \le |x_1|$ and $\sgn(b_{n-1} +k_{n-1} x_1) \ne \sgn(c_{n})$ where $$\sgn(x) \defeq \begin{cases}
        1 \; &\text{if} \; x > 0, \\
        0 \; &\text{if} \; x = 0, \\
        -1 \; &\text{if} \; x < 0. \\
    \end{cases}$$ 
    Let $c_{n-1} = b_{n-1} +k_{n-1} x_1$. Then, we have $$| c_{n-1} x_{n-1}+ c_n x_n | \le \max( |c_{n-1} x_{n-1}|,|c_{n} x_{n}|) \le |x_1||x_n|.$$ Recursively, for $j=n-2, \dots , 2$, pick $k_j \in \bbZ$ such that $|b_{j} + k_jx_1 | \le |x_1|$ and $\sgn(b_{j} + k_jx_1 ) \ne \sgn(c_{j+1})$. Let $c_j = b_{j} + k_jx_1$. Then, it holds that $ |\sum_{i=j}^n c_i x_i|  \le \max(|c_jx_j|, |\sum_{i=j+1}^n c_i x_i| ) \le |x_1||x_n|.$ Finally, let $c_1 = b_1 - \sum_{i=2}^n k_ix_i$. Then, we have 
    \begin{align*}
         \sum_{i=1}^n c_i x_i =  \sum_{i=1}^n b_i x_i = 1.
    \end{align*}
    Moreover, since $|c_1x_1| = |1-\sum_{i=2}^n c_i x_i| \le 1+ |\sum_{i=2}^n c_i x_i| \le 1+ |x_1||x_n| < |x_1|(1+|x_n|)$ we have $|c_1| \le |x_n|$.
\end{proof}

\subsection{Proof of \cref{thm:param_count}}\label{sec:pfthm:param_count}

In this proof, we use \cref{lem:param_count,lemma:bchange}. 
\begin{proof}
Let $\delta = \omega_{f^*}^{-1}(\varepsilon)$. 
First, consider the case that $\delta \ge 2q_{\max}$. 
Choose $\bfz^*\in\Qps^d$ so that 
$$|f^*(\bfz^*)-\round{f^*(\bfz^*)}|=\min_{\bfx\in\Qps^d}|f^*(\bfx)-\round{f^*(\bfx)}|.$$
Let $f(\bfx;\Qps) = \round{f^*(\bfz^*)}$, i.e., $f$ is a one-layer $\sigma$ quantized network with a single bias parameter $\round{f^*}(\bfz^*)$.
Then, for each $\bfx\in\Qps^d$,
\begin{align*}
\left|f(\bfx;\Qps)-f^*(\bfx)\right|& \le \left|\round{f^*(\bfz^*)}-f^*(\bfz^*)\right| + \left|f^*(\bfz^*)-f^*(\bfx)\right| \\
&\le 
|f^*(\bfx)-\round{f^*(\bfx)}| + \varepsilon
\end{align*}
where the second inequality is from the definition of $z^*$ and the fact that $\|z^*-\bfx\|_\infty\le 2 q_{\max} \le \delta$.

Next, consider the case that $ \frac{1}{s} < \delta < 2q_{\max}$.  Let $N = \min\{ n \in \bbN : n \ge \frac{2 q_{\max}}{\delta} \}$ and 
$\mathcal{G} = \{ -q_{\max} + i \delta : i \in \{0\} \cup [N]\}$. From these definitions, one can observe that $|\mathcal{G}| = N + 1  \le \frac{2 q_{\max} }{\delta} + 2$.
For each $\bfp = (p_1, \dots , p_d) \in \mathcal{G}^d$, we define the set $\mathcal{C}_{\bfp}$ as 
\begin{align*}
    \mathcal{C}_{\bfp} \defeq  \{ (x_1, \dots , x_d) \in \Qps^d :  p_i \le x_i < p_i + \delta \quad \forall i \in [d] \}. 
\end{align*}
Furthermore, for each $\bfp\in\mcG^d$, we define $\tilde{\bfp} \in \mcC_{\bfp}$ as 
\begin{align*}
    \tilde{\bfp} \defeq \argmin_{\bfx \in \mcC_{\bfp}} \left| \round{ f^*(\bfx) } - f^*(\bfx) \right|;
\end{align*}
if there are multiple such $\tilde{\bfp}$, we choose an arbitrary one.
Then, for any $\bfx \in \mathcal{C}_{\bfp}$, it holds that $ \| \bfx-\tilde{\bfp} \|_{\infty} \le \delta$. 

Let $\tilde{\mcG}^d \defeq \{ \tilde{\bfp} : \bfp \in \mcG^d  \}$.
By \cref{lemma:bchange}, there exists $b' \in \Qps$ such that $|b'| \le q_{\max}$, and  $\mcS_{\sigma,p,s,b'}=\{ \round{\gamma + b'} : \gamma \in \mcS_{\sigma,p,s}^\circ, | \gamma| \le 2 q_{\max} + \frac{3}{2s}  \}=\Qps$.
Hence, for each $\tilde{\bfp} \in \tilde{\mcG}^d$, %
there exists $\gamma_{\tilde{\bfp}} \in \mcS_{\sigma,p,s}^\circ $ such that $\round{\gamma_{\tilde{\bfp}}+b'} = \round{f^\ast(\tilde{\bfp})}$ with $|\gamma_{\tilde{\bfp}}| \le 2 q_{\max}+\frac{3}{2s}$. 
 By \cref{lem:indicator}, there exists a four-layer $\sigma$ quantized network $f_{\tilde{\bfp}} : \Qps^d \to \Qps$ (without rounding the output) such that 
\begin{align*}
    f_{\tilde{\bfp}}(\bfx) = \gamma_{\tilde{\bfp}}  \times \indc{\mathcal{C}_{\bfp}}{\bfx}%
\end{align*}
for all $\bfx \in \Qps^d$. Here, the number of parameters in $f_{\tilde{\bfp}}$ is {at most $O\left(2^{2p}s^2d\right)$} by \cref{lem:param_count}. Since the collection $\{ \mathcal{C}_{\bfp} \}_{\bfp \in \mathcal{G}^d}$ is a partition of $\Qps^d$, we can construct a four-layer $\sigma$ quantized network $f(\bfx;\Qps)$ such that 
\begin{align*}
    f\left(\bfx;\Qps\right) & \defeq \round{b'+\sum_{\tilde{\bfp} \in \tilde{\mcG}^d} \gamma_{\tilde{\bfp}}  \times \indc{\mathcal{C}_{\bfp}}{\bfx}}=  \sum_{\tilde{\bfp}_\bfx \in \tilde{\mcG}^d} \round{f^\ast(\tilde{\bfp}_{\bfx})}  \times \indc{\mathcal{C}_{\bfp}}{\bfx}, 
\end{align*}
for all $\bfx \in \Qps^d$. Here, the number of parameters in $f$ is {$O\left(2^{2p}s^2d |\mathcal{G}^d|\right)= O\left(2^{2p}s^2d (2q_{\max}+2\delta)^d\delta^{-d}\right) =O\left(2^{2p}s^2d (6q_{\max})^d\delta^{-d}\right) $} where we use $|\mcG^d|\le\frac{2q_{\max}}\delta+2$. \\
From our construction of $f$, for any $\bfx \in \Qps^d$, there exists $\bfp_{\bfx} \in \mathcal{G}^d$ such that $\bfx \in \mcC_{\bfp_{\bfx}}$ and $ f\left(\bfx;\Qps\right) = \round{f^\ast(\tilde{\bfp}_{\bfx})}$. Furthermore, it holds that
\begin{align*}
    | f^*(\bfx) - f^*(\tilde{\bfp}_{\bfx})| \le \omega_{f^*} \left( \| \bfx - \tilde{\bfp}_{\bfx}\|_\infty \right) \le \omega_{f^*}(\delta) = \varepsilon.
\end{align*}
This implies that
\begin{align*}
    \left| f(\bfx;\Qps) - f^*(\bfx) \right| &\le \left| f\left(\bfx;\Qps\right) - f^*(\tilde{\bfp}_{\bfx}) \right| + \left| f^*(\tilde{\bfp}_{\bfx}) - f^*(\bfx) \right| \nonumber \\
    &\le \left| \round{f^\ast(\tilde{\bfp}_{\bfx})} - f^*(\tilde{\bfp}_{\bfx}) \right| + \varepsilon\\ &\le  \left| \round{ f^*(\bfx) } - f^*(\bfx) \right| + \varepsilon 
\end{align*}
where the last inequality follows from the definition of $\tilde{\bfp}_{\bfx}$.

We lastly consider the remaining case that $\omega_{f^*}^{-1}(\varepsilon) \le 1/s$. %
For each $\bfx \in \Qps^d$, by \cref{lem:indicator,lem:param_count}, we can construct a four-layer $\sigma$ quantized  network $f(\bfx;\Qps)$ such that  
\begin{align*}
    f(\bfx;\Qps) = \sum_{\bfp \in \Qps^d} \round{f^\ast(\bfp)}  \times \indc{\{ \bfp \}}{\bfx},
\end{align*}
for all $\bfx \in \Qps^d$. Then we have 
\begin{align*}
    \left| f(\bfx;\Qps) - f^*(\bfx)\right| = \left| f(\bfx;\Qps) - \round{f^*(\bfx)} \right| .
\end{align*}
Since $|\Qps|=O(2^{p+1})$,  the number of parameters in $f$ is at most {$O\left(2^{2p}s^2d \left|\Qps^d\right|\right)$ $= O\left(2^{d(p+1)+2p}s^2d\right) $}. 
\end{proof}

\section{Conclusion}\label{sec:conclusion}
In this paper, we study the expressive power of quantized networks under fixed-point arithmetic. We provide a necessary condition and a sufficient condition on activation functions and $\Qps$ for the universal representation of quantized networks. 
We compare our results with classical universal approximation theorems and show that popular activation functions and fixed-point arithmetic are capable of universal representation.
We further extend our results to quantized networks with binary weights. %
We believe that our findings offer insights that can enhance the understanding of quantized network theory.

\backmatter

\begin{appendices}

\end{appendices}

\bibliography{reference}%

\end{document}